\documentclass{article}

\usepackage{microtype}
\usepackage{graphicx}
\usepackage{booktabs} 
\usepackage{xcolor}
\usepackage{colortbl}
\usepackage{makecell}
\usepackage{enumitem}
\usepackage{subcaption}
\usepackage{tocbibind}
\usepackage{appendix}
\usepackage{tocloft}
\definecolor{LightCyan}{rgb}{0.58, 0.8, 0.85}
\definecolor{LightGray}{rgb}{0.9,0.9,0.9}

\usepackage{hyperref}

\usepackage[accepted]{icml2024}

\usepackage[compact]{titlesec}
\titlespacing*{\section}{0pt}{*0.5}{*0.5}
\titlespacing*{\subsection}{0pt}{*0.3}{*0.3}
\titlespacing*{\subsubsection}{0pt}{*0.3}{*0.3}

\usepackage{amsmath}
\usepackage{amssymb}
\usepackage{mathtools}
\usepackage{amsthm}

\usepackage[capitalize,noabbrev]{cleveref}

\theoremstyle{plain}
\newtheorem{theorem}{Theorem}[section]

\newtheorem{lemma}[theorem]{Lemma}
\newtheorem{corollary}[theorem]{Corollary}
\theoremstyle{definition}
\newtheorem{definition}[theorem]{Definition}
\newtheorem{assumption}[theorem]{Assumption}
\theoremstyle{remark}
\newtheorem{remark}[theorem]{Remark}

\usepackage[textsize=tiny,disable]{todonotes}

\usepackage{smile}
\newcommand{\todoq}[2][]{\todo[size=\scriptsize,color=orange!20!white,#1]{Quanquan: #2}}

\icmltitlerunning{Uncertainty-Aware Reward-Free Exploration with General Function Approximation}

\newcommand{\alg}{$\texttt{GFA-RFE}$}

\begin{document}

\twocolumn[
\icmltitle{Uncertainty-Aware Reward-Free Exploration with General Function Approximation}



\icmlsetsymbol{equal}{*}

\begin{icmlauthorlist}
\icmlauthor{Junkai Zhang}{equal,ucla}
\icmlauthor{Weitong Zhang}{equal,ucla}
\icmlauthor{Dongruo Zhou}{iub}
\icmlauthor{Quanquan Gu}{ucla}
\end{icmlauthorlist}

\icmlaffiliation{ucla}{Department of Computer Science, University of California, Los Angeles, California, USA}
\icmlaffiliation{iub}{Department of Computer Science, Indiana University Bloomington, Indiana, USA}

\icmlcorrespondingauthor{Quanquan Gu}{qgu@cs.ucla.edu}
\icmlkeywords{Machine Learning, ICML}

\vskip 0.3in
]



\printAffiliationsAndNotice{\icmlEqualContribution} 

\begin{abstract}
Mastering multiple tasks through exploration and learning in an environment poses a significant challenge in reinforcement learning (RL). Unsupervised RL has been introduced to address this challenge by training policies with intrinsic rewards rather than extrinsic rewards. However, current intrinsic reward designs and unsupervised RL algorithms often overlook the heterogeneous nature of collected samples, thereby diminishing their sample efficiency.
To overcome this limitation, in this paper, we propose a reward-free RL algorithm called \alg. The key idea behind our algorithm is an \emph{uncertainty-aware intrinsic reward} for exploring the environment and an \emph{uncertainty-weighted} learning process to handle heterogeneous uncertainty in different samples. Theoretically, we show that in order to find an $\epsilon$-optimal policy, \alg~needs to collect $\tilde{O} (H^2 \log N_{\cF} (\epsilon) \dim (\cF) / \epsilon^2 )$ number of episodes, where $\cF$ is the value function class with covering number $N_{\cF} (\epsilon)$ and generalized eluder dimension $\dim (\cF)$. Such a result outperforms all existing reward-free RL algorithms. We further implement and evaluate \alg~across various domains and tasks in the DeepMind Control Suite.  Experiment results show that \alg~outperforms or is comparable to the performance of state-of-the-art unsupervised RL algorithms.
\end{abstract}



%




\section{Introduction}



Deep reinforcement learning (RL) has been the source of many breakthroughs in games (e.g., Atari game \citep{mnih2013playing} and Go game \citep{silver2016mastering}) and robotic control \citep{levine2016end}\todoq{add reference} over the last ten years. A key component of RL is exploration, which requires the agent to explore different states and actions before finding a near-optimal policy. Traditional exploration strategy involves iteratively executing a policy guided by a specific reward function, limiting the trained agent to solving only the single task for which it was trained. Designing an efficient exploration strategy agnostic to reward functions is crucial, as it prevents the agent from repeated learning under different reward functions, thereby avoiding inefficiency and potential intractability in sample complexity.


\begin{figure*}[t!]
\centering
\includegraphics[width=0.8\textwidth]{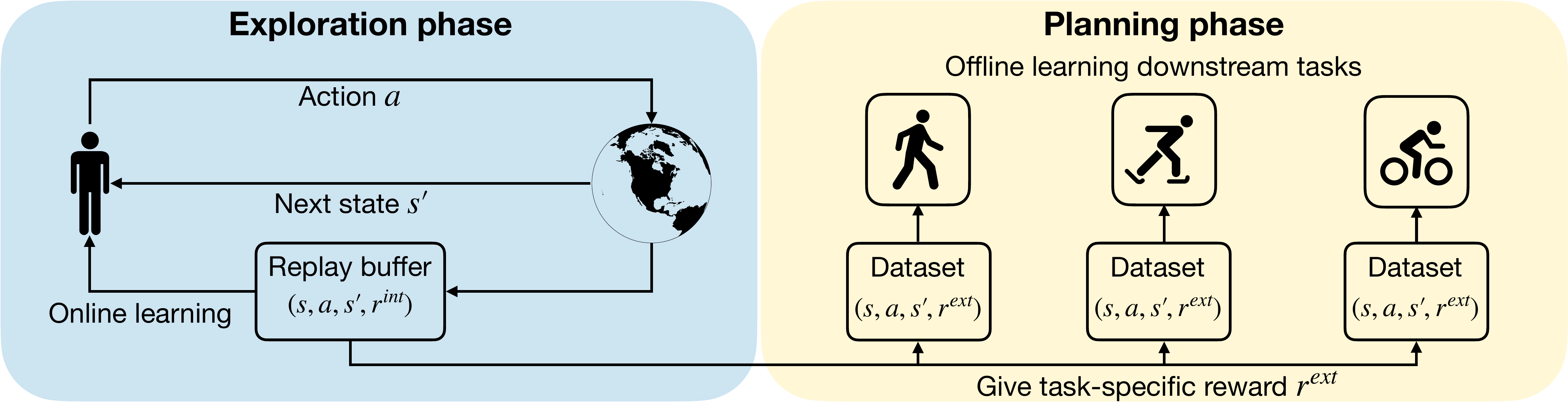}
\caption{The overall framework of reward-free exploration.}
\label{fig:framework}
\vspace{-1em}
\end{figure*}

To achieve this goal, \citet{jin2020reward} introduced a two-phase RL framework known as ``reward-free exploration" (RFE) for the basic tabular MDP setting. In this framework, the agent only interacts with the environment in the first phase without reward. Upon receiving the specific reward in the second phase, the algorithm returns a near-optimal policy without further interactions. The overall framework is displayed in Figure~\ref{fig:framework}. A series of subsequent works extended the idea to more complex settings, such as linear MDPs \citep{wang2020reward, zanette2020provably,wagenmaker2022reward,hu2022towards} and linear mixture MDPs \citep{zhang2021reward, chen2022nearoptimal, zhang2023optimal}. \todoq{miss your own paper?} RFE diverges from classical RL approaches by not relying on a specific reward function for exploration. Instead, RFE utilizes an ``intrinsic reward'', a.k.a., pseudo-reward function, defined based on all previously explored samples. This encourages the agent to venture into unexplored states and actions. In particular, in the realm of deep RL where no structural assumptions are made, recent studies \citep{pathak2017curiosity,burda2018exploration,eysenbach2018diversity,lee2019efficient,pathak2019self,liu2021aps,liu2021behavior} \todoq{can you verify all these works are unsupervised RL?} have developed RFE (a.k.a., unsupervised RL) algorithms by employing various intrinsic reward functions, demonstrating promising performance in finding the near-optimal policy.

Despite the success of intrinsic reward functions in facilitating RFE, the design of these functions in prior studies could be further optimized. For example, \citet{kong2021online} defined an intrinsic reward based on the maximum difference between function pairs that show similarity in past data. This approach essentially treats each collected sample equally. It is a well-established principle in RL that in order to achieve optimal sample efficiency, different samples should be treated distinctively based on their importance. Notably, \citet{zhang2023optimal} utilized variance-dependent weights to address the heteroscedasticity observed in samples, thereby achieving optimal sample complexity in linear mixture MDPs. However, this approach calculates its intrinsic reward by nested iterative optimization, which hampers computational efficiency and practical applicability. Therefore, for RFE or more generally unsupervised RL, we are faced with the following question:
\begin{center}
    Is it possible to craft an intrinsic reward function that excels both theoretically and empirically?
\end{center}

We answer the above question affirmatively by proposing a variance-adaptive intrinsic reward for RFE. Theoretically, we show that our method enjoys a finite sample complexity in finding the near-optimal policy for any given reward, and our theoretical guarantee is tighter than existing methods. Empirically, we show that by incorporating variance information, a series of existing reward-free RL baselines can be further improved in terms of sample efficiency. The main contributions of our work can be summarized as follows.
\begin{itemize}[leftmargin=*]
    \item We propose a new algorithm \alg~under the RFE framework. The key innovation of \alg~is a new intrinsic reward, which depends on an uncertainty estimation of each past state-action pair that appeared during the exploration phase. Our designed intrinsic reward relies more on observed samples with lower uncertainty, and it encourages the agent to explore states and actions with larger uncertainty. Intuitively speaking, such a strategy ensures the agent to find samples that are universally suitable for all reward functions that it may encounter during the planning phase, without knowing the extrinsic rewards in hindsight. 
    \item Theoretically, we prove that during the planning phase, given any reward function $r$, \alg~achieves an $\tilde{O} (H^2 \log N_{\cF} (\epsilon) \dim (\cF) / \epsilon^2 )$ sample complexity to find the $\epsilon$-optimal policy w.r.t. the reward $r$, where $\cF$ is the value function class with covering number $N_{\cF} (\epsilon)$ and generalized eluder dimension $\dim (\cF)$. Our sample complexity outperforms the existing sample complexity result achieved by \citet{kong2021online} (see Table \ref{tab:1}), which verifies our claim that an adaptive intrinsic reward improves exploration efficiency. 
    \item We also show that our variance-adaptive intrinsic reward can efficiently explore the environment in practice through extensive experiments on the DeepMind Control Suite \citep{tassa2018deepmind}. Our theory-guided algorithm \alg~ exhibits compatible or superior performance compared with the state-of-the-art unsupervised exploration methods. This promising result demonstrates the huge potential of incorporating the theories into practice to solve real-world problems.
\end{itemize}



\noindent\textbf{Notation}
We denote by $[n]$ the set $\{1,\cdots, n\}$. For two positive sequences $\{ 
 a_n\}$ and $\{ b_n \}$ with $n = 1,2,\cdots$, we write $a_n = O(b_n) $ if there exists an absolute constant $C>0$ such that $a_n \le C b_n$ holds for all $n\ge 1$, write $a_n = \Omega(b_n) $ if there exists an absolute constant $C>0$ such that $a_n \ge C b_n$ holds for all $n \ge 1$, and write $a_n = o(b_n)$ if $a_n/b_n \rightarrow 0$ as $n\rightarrow\infty$. We use $\tilde{O}(\cdot)$ and $\tilde{\Omega}(\cdot)$ to further hide the polylogarithmic factors.

\section{Related Work}
\newcolumntype{g}{>{\columncolor{LightCyan!40}}c}
\begin{table*}[!ht]
\centering
\caption{Comparison of episodic reward-free RL algorithms in different settings. Column \textbf{Comp. Eff.} (Computational Efficiency) indicates if the algorithm can be efficiently implemented, or can be segregated by some efficient algorithm. Column \textbf{Time Homo.} (Time Homogeneity) indicates if the setting is time homogeneous $\checkmark$ or time inhomogeneous $\times$. Time inhomogeneous settings usually yield an additional $H$ in sample complexity in learning different stages $h \in [H]$. The sample complexity is evaluated under the reward scale $r_h(s_h, a_h) \in [0, 1]$. The results with bounded total reward assumption ($\sum_{h=1}^H r_h(s_h, a_h) \le 1$) are translated by inserting an additional $H^2$ dependency on the reward scale and marked with $(^*)$. Dimension $d$ in general function approximations inherits their original definitions in the paper, and it usually corresponds to the dimension of linear function when reduced to linear function approximations. The row with a light cyan background indicates our results.}
\vspace{0.5em}
\label{tab:1}
\small{
\begin{tabular}{cgggg}
\toprule
\rowcolor{white}
Setting & Algorithm & \Gape[0pt][2pt]{\makecell[c]{Comp. Eff.}} & \Gape[0pt][2pt]{\makecell[c]{Time Homo.}} & Sample Complexity \\
\hline
\rowcolor{LightGray}
 \cellcolor{white} & \citet{wang2020reward}  &$\checkmark$& $\times$&$\widetilde{O}\left(H^6 d^3 \epsilon^{-2}\right)$ \\
\rowcolor{white}
 \cellcolor{white}\makecell[c]{\cellcolor{white}Linear \\ \cellcolor{white}MDP} & \Gape[0pt][2pt]{\makecell[c]{FRANCIS \\ \citep{zanette2020provably}}}  &$\checkmark$& $\times$&$\widetilde{O}\left(H^5 d^3 \epsilon^{-2}\right)$ \\
 \rowcolor{LightGray}
 \cellcolor{white} & \makecell[c]{RFLIN \\ \citep{wagenmaker2022reward}}  &$\checkmark$& $\times$ &$\widetilde{O}\left(H^5 d^2 \epsilon^{-2}\right)$ \\
 \rowcolor{white}
 & \Gape[0pt][2pt]{\makecell[c]{LSVI-RFE \\ \citet{hu2022towards}}} & $\checkmark$ & $\times$ & $\tilde \cO(H^4 d^2 \epsilon^{-2})$\\
\hline
\rowcolor{LightGray}
\cellcolor{white} \makecell[c]{\cellcolor{white} \\ \cellcolor{white} Linear} \vspace{-0.3em}& \Gape[0pt][2pt]{\makecell[c]{UCRL-RFE+ \\ \citep{zhang2021reward}} }&$\checkmark$& \checkmark &$\widetilde{O}\left(H^4d(H+d)\epsilon^{-2}\right)$ \\
\rowcolor{white}
 \cellcolor{white} Mixture & \citet{chen2022nearoptimal} &$\times$& $\times$ &  $\widetilde{O}\left(H^3d(H+d)\epsilon^{-2}\right)$ \\
\rowcolor{LightGray}
 \cellcolor{white}\Gape[0pt][2pt]{\makecell[c]{ \cellcolor{white} MDP\\ \cellcolor{white}}    } & \Gape[0pt][2pt]{\makecell[c]{HF-UCRL-RFE++ \\ \citep{zhang2023optimal}}}&$\times$& \checkmark &$\tilde{O}(H^2d^2\epsilon^{-2})^*$ \\
\hline
\rowcolor{white}
\cellcolor{white} & \citet{kong2021online} & $\checkmark$& $\times$ & $\tilde \cO(H^6d^4\epsilon^{-2})$\\
\rowcolor{LightGray}
\cellcolor{white} \Gape[0pt][2pt]{\makecell[c]{\cellcolor{white} General \\ \cellcolor{white}Function}} \vspace{-0.3em}& \Gape[0pt][2pt]{\makecell[c]{Reward-Free E2D \\ \citep{chen2022unified}}}  & $\times$ & $\times$ & $\tilde \cO(d\log |\cP|\epsilon^{-2})$\\
\rowcolor{white}
\cellcolor{white}\Gape[0pt][2pt]{\makecell[c]{\cellcolor{white} Approximation\\ \cellcolor{white}}} & \Gape[0pt][2pt]{\makecell[c]{RFOlive \\ \citep{chen2022statistical}}} & $\times$ & $\times$ & $\tilde \cO(\mathrm{poly}(H) d^2_{\text{BE}}\log(|\cF||\cR|)\epsilon^{-2})$\\
\cellcolor{white} & \alg~\textbf{(Ours)} & $\checkmark$ & $\times$ & $\tilde \cO(H^4d_{K, \delta}^2\epsilon^{-2})^*$\\
\hline
\rowcolor{white} 
\cellcolor{white}\makecell[c]{\cellcolor{white}Linear \\ \cellcolor{white}(Mixture) MDP} & \Gape[0pt][2pt]{\makecell[c]{Lower bound \\ \citep{hu2022towards}}} & N/A & $\times$ & $\tilde \Omega(H^3d^2\epsilon^{-2})$\\
\bottomrule
\end{tabular}
}
\vspace{-2em}
\end{table*}

\noindent\textbf{Reinforcement learning with general function approximation.}
RL with general function approximation has been widely studied in recent years, due to its ability to describe a wide range of existing RL algorithms. To explore the theoretical limits of RL and understand the practical DRL algorithms, various statistical complexity measurements for general function approximation have been proposed and developed. For instance, Bellman rank~\citep{jiang2017contextual}, Witness rank~\citep{sun2019model}, eluder dimension~\citep{russo2013eluder}, Bellman eluder dimension~\citep{jin2021bellman}, Decision-Estimation Coefficient (DEC)~\citep{foster2021statistical}, Admissible Bellman Characterization~\citep{chen2022general}, generalized eluder dimension \citep{agarwal2022vo}, etc. Among different statistical complexity measurements, \citet{foster2021statistical} showed a DEC-based lower bound of regret which holds for any function class. Specifically, our algorithm falls into the category of generalized eluder dimension function class, which includes linear MDPs \citep{jin2020provably} as its special realization.

\noindent\textbf{Reward-free exploration.}
Unlike standard RL settings where the agent interacts with the environment with reward signals, \emph{reward-free exploration}~\citep{jin2020reward} in RL introduced a two-phase paradigm. In this approach, the agent initially explores the environment without any reward signals. Then, upon receiving the reward functions, it outputs a policy that maximizes the cumulative reward, without any further interaction with the environment. \citet{jin2020reward} first achieved $\tilde{O}(H^5S^2A/\epsilon^2)$ sample complexity in tabular MDPs by executing exploratory policy visiting states with probability proportional to its maximum visitation probability under any possible policy. Subsequent works~\citep{kaufmann2021adaptive, menard2021fast} proposed algorithms RF-UCRL and RF-Express to gradually improve the result to  $\tilde{O}\left(H^3S^2A\epsilon^{-2}\right)$. The optimal sample complexity bound $\widetilde{O}(H^2S^2A\epsilon^{-2})$ was achieved by algorithm SSTP proposed in \citet{zhang2020nearly}, which matched the lower bound provided in \citet{jin2020reward} up to logarithmic factors. Recent years have witnessed a trend of reward-free exploration in RL with function approximations, while most of these works are considering linear function approximation: in the linear MDP setting,~\citet{wang2020reward} propose an exploration-driven reward function and the minimax optimal bound was achieved by~\citet{hu2022towards} by introducing the weighted regression in the algorithm. In linear mixture MDPs,~\citet{zhang2021model} proposed the `pseudo reward' to encourage exploration,~\citet{chen2022nearoptimal, wagenmaker2022reward} improved the sample complexity by introducing a more complicated, recursively defined pseudo reward. The minimax optimal sample complexity, $\tilde \cO(d^2 / \epsilon^2)$ was achieved by~\citet{zhang2023optimal} in the horizon-free setting. Moving forward, in the general function approximation setting,~\citet{kong2021online} used `online sensitivity score' to estimate the information gain thus providing a $\tilde \cO(d^4H^6\epsilon^{-2})$ sample complexity where $d$ is the dimension of contexts when reduced to linear function approximations. Yet another line of works~\citep{chen2022unified, chen2022statistical} aimed to follow the Decision-Estimation Coefficient (DEC, \citealt{foster2021statistical}) and provided a unified framework for reward-free exploration with general function approximations, achieving a $\tilde \cO(\text{ploy}(H)d^2\epsilon^{-2})$, nevertheless, all existing works with general function approximations leave a huge gap between their proposed upper bound and lower bound, even when reduced to linear settings. We record existing results in Table~\ref{tab:1}.

\noindent\textbf{Unsupervised reinforcement learning.}
Witnessing recent advancements in unsupervised CV and NLP tasks, unsupervised reinforcement learning has emerged as a new paradigm trying to learn the environment without supervision or reward signals. As suggested in~\citet{laskin2021urlb}, these works are mainly separated into two lines: unsupervised representation learning in RL and unsupervised behavioral learning.

Unsupervised representation learning in RL mainly addresses issues on how to learn good representations for different states $s$, which can facilitate efficient learning of a policy $\pi(a | s)$. From the theoretical side, a list of works have identified how to select or learn good representations for various RL tasks with linear function approximations, by using MLE~\citep{uehara2021representation}, contrastive learning~\citep{qiu2022contrastive} or model selection~\citep{papini2021reinforcement, zhang2021provably}. From the empirical side, various methods in unsupervised learning or self-supervised learning are applied to RL tasks, including contrastive learning~\citep{laskin2020curl,stooke2021decoupling,yarats2021reinforcement}, autoencoders~\citep{yarats2021improving} and world models~\citep{hafner2019dream,hafner2019learning}.

Unsupervised behavioral learning in RL aims to eliminate this reward signal during exploration. Therefore, the agent can be adapted to different tasks in the downstream fine-tuning. To replace the `extrinsic' reward signals, these methods usually leverage different `intrinsic rewards' during exploration. Many recent algorithms have been proposed to learn from different types of intrinsic reward, which is based on the prediction, information gain or entropy. In particular, \citet{pathak2017curiosity,burda2018large,pathak2019self} are referred to as ``knowledge based'' intrinsic reward~\citep{laskin2021urlb} and they all maintain a neural network $g(s_t, a_t)$ to predict the next state $s_{t+1}$ from the current state and actions. Among these three methods, \citet{pathak2017curiosity} and \citet{burda2018large} are using the prediction error $|s_{t+1} - g(s_t, a_t)|$ as the intrinsic reward, \citet{pathak2019self} is using the variance of $N$ ensemble neural networks (i.e., $\text{Var}~g_i(s_t, a_t)$) as the intrinsic reward. On the other hand, \citet{lee2019efficient, eysenbach2018diversity, liu2021aps} are trying to maximize the mutual information to complete the exploration of the agent, thus they are referred to as the ``complete-based'' algorithms. In addition, \citet{liu2021behavior} is trying to maximize the entropy of the collected observations via a kNN method, which is referred to as the “data-based” algorithm. URLB~\citep{laskin2021urlb} provided a unified framework providing benchmarks for all these intrinsic rewards.  
\section{Problem Setup}
\subsection{Time-Inhomogeneous Episodic MDPs}

We model the sequential decision making problem via time-inhomogeneous episodic Markov decision processes (MDPs), which can be denoted as tuple $\cM = (\cS,\cA, H, \PP = \{ \PP_h\}_{h=1}^H, r=\{r_h\}_{h=1}^H)$ by convention. Here, $\cS$ and $\cA$ are state and action spaces, $H$ is the length of each episode, $\PP_h: \cS \times \cA \times \cS \to [0, 1]$ is the transition probability function at stage $h$ for state $s$ to transit to state $s'$ after executing action $a$, and $r_h: \cS \times \cA \to [0, 1]$ is the deterministic reward function at stage $h$. For any policy $\pi=\{\pi_h\}_{h=1}^{H}$, reward $r=\{r_h\}_{h=1}^H$, and stage $h\in[H]$, the value function $V_h^{\pi}(s;r)$ and the state-action value function $Q_h^{\pi}(s,a;r)$ is defined as:
\begin{align}
Q^{\pi}_h(s,a;r) &= \EE\bigg[ \sum_{h'=h}^H r_{h'}\big(s_{h'}, 
a_{h'}\big)\bigg|s_h=s,a_h=a,\notag \\ 
&  s_{h'+1}\sim \PP_{h'}(\cdot|s_{h'},a_{h'}),a_{h'+1}=\pi(s_{h'+1})\bigg], \notag\\
V_h^{\pi}(s;r) &= Q_h^{\pi}(s, \pi_h(s);r).\notag
\end{align}
Furthermore, the optimal value function $V_h^*(s;r) $ is defines as $ \max_{\pi}V_h^{\pi}(s;r)$, and the optimal action-value function $Q_h^*(s,a;r) $ is defined as $\max_{\pi}Q_h^{\pi}(s,a;r)$. For simplicity, we utilize the following bounded total reward assumption:
\begin{assumption}\label{assumption:bounded_reward}
    The total reward for every possible trajectory is assumed to be within the interval of $(0,1)$. 
\end{assumption}
Up to rescaling, Assumption~\ref{assumption:bounded_reward} is more general than the standard reward scale assumption where $r_{h} \in [0,1]$ for all $h\in [H]$. Assumption~\ref{assumption:bounded_reward} also ensures that the value function $V^\pi_h(s)$ and action-value function $Q_h^{\pi}(s,a;r)$ belong to the interval $[0,1]$. 

For any function $V: \cS \to \RR$ and stage $h \in [H]$, the first-order Bellman operator $\cT_h$ is defined as:
\begin{align*}
     \cT_h V(s, a;r) & = \EE_{ s'\sim \PP(\cdot |s,a) }\Big[r_h(s,a) + V(s';r)\Big].
\end{align*}
For simplicity, we further define the shorthand:
\begin{align*}
    [\PP_h V](s,a;r) &= \EE_{s' \sim \PP_h(\cdot|s,a)}V(s';r), \\
    [\VV_h V](s,a;r) &= [\PP_h V^2 ](s,a;r) - [\PP_h V]^2(s,a;r).
\end{align*}
Throughout the paper, if the reward $r$ is clear in the context, we omit the notation $r$ in $Q$ and $V$ for simplicity.

\subsection{Reward-free Exploration}
\noindent\textbf{Reward-free RL.} In reward-free RL, the real reward function is accessible only after the agent finishes the interactions with the environment. Specifically, the algorithm can be separated into two phases: (i) \textit{Exploration phase}: the algorithm can't access the reward function but collects $K$ episodes of samples by interacting with the environment.  (ii) \textit{Planning phase}: The algorithm is given reward function $\{r_h\}_{h=1}^H$ and is expected to find the optimal policy without interaction with the environment.

To deal with the randomness in learning processes and evaluate the efficiency of algorithms, we adopt the commonly used $(\epsilon, \delta)$-learnability concept, which is formulated in Definition~\ref{def:learnability}.
\begin{definition}\label{def:learnability}
 ($(\epsilon, \delta)$-learnability). Given an MDP transition kernel set $\mathcal{P}$, reward function set $\mathcal{R}$ and a initial state distribution $\mu$, we say a reward-free algorithm can $(\epsilon, \delta)$-learn the problem $(\mathcal{P}, \mathcal{R})$ with sample complexity $K(\epsilon, \delta)$, if for any transition kernel $P \in \mathcal{P}$, after receiving $K(\epsilon, \delta)$ episodes in the exploration phase, for any reward function $r \in \mathcal{R}$, the algorithm returns a policy $\pi$ in planning phase, such that with probability at least $1-\delta,~V_1^*\left(s_1 ; r\right)-V_1^\pi\left(s_1 ; r\right) \leq \epsilon$.
\end{definition}

\subsection{General Function Approximation}
In this work, we focus on the model-free value-based RL methods, which require us to use a predefined function class to estimate the optimal value function $Q^*_h(s,a; r)$ for any reward $r$. We use $\cF := \{\cF_h\}_{h = 1}^H $ to denote the function class we will use during all $H$ stages. To build the statistical complexity of using $\cF$ to learn $Q^*_h(s,a; r)$, we require several assumptions and definitions that characterize the cardinality of the function class.

\begin{assumption}[Completeness, \citet{zhao2023nearly}] \label{assumption:complete}
   Given $\cF := \{\cF_h\}_{h = 1}^H $ which is composed of bounded functions $f_h: \cS \times \cA \to [0, L]$. We assume that for any $h$ and function $V: \cS \to [0, 1]$ and $r: \cS \times \cA\to [0,1]$, there exist $f_1, f_2 \in \cF_h$ such that for any $(s, a) \in \cS \times \cA$, \begin{align*} 
    f_1(s, a) &= \EE_{s' \sim \PP_h(\cdot|s, a)}\big[r(s, a) + V(s')\big] , \\ 
    f_2(s, a) &= \EE_{s' \sim \PP_h(\cdot|s, a)}\Big[\big(r(s, a) + V(s')\big)^2\Big].
   \end{align*}
   We assume that $L = O(1)$ throughout the paper. 
\end{assumption}

\begin{definition}[Generalized eluder dimension, \citealt{agarwal2022vo}] \label{def:ged}
   Let $\lambda \ge 0$ and $h\in[H]$, a sequence of state-action pairs $Z_h = \{z_{i,h} = (s_{h}^i,a_{h}^i)\}_{i \in [K]}$ and a sequence of positive numbers $\bsigma_h = \{\sigma_{i,h}\}_{i \in [K]}$. The generalized eluder dimension of a function class $\cF_h: \cS \times \cA \to [0, L]$ with respect to $\lambda$ is defined by $\dim_{\alpha, K}(\cF_h) := \sup_{Z_h, \bsigma_h:|Z_h| = K, \bsigma_h \ge \alpha} \dim(\cF_h, Z_h, \bsigma_h)$ 
   \begin{align*} 
       \dim&(\cF_h,  Z_h, \bsigma_h) := \\
       & \sum_{i=1}^K \min \bigg(1, \frac{1}{\sigma_i^2} D_{\cF_h}^2(z_{i,h}; z_{[i - 1],h}, \sigma_{[i - 1],h})\bigg), \\
       D_{\cF_h}^2&(z;  z_{[i - 1],h}, \sigma_{[i - 1],h}) := \\
        & \sup_{f_1, f_2 \in \cF_h} \frac{(f_1(z) - f_2(z))^2}{\sum_{s \in [i - 1]} \frac{1}{\sigma_{s,h}^2}(f_1(z_{s,h}) - f_2(z_{s,h}))^2 + \lambda}.
   \end{align*}
   We write $\dim_{\alpha, K}(\cF) := H^{-1} \cdot \sum_{h \in [H]}\dim_{\alpha, K}(\cF_h)$ for short when $\cF$ is a collection of function classes $\cF = \{\cF_h\}_{h = 1}^H$ in the context.
\end{definition}
\begin{remark}
    \citet{kong2021online} introduced a similar definition called ``sensitivity". In particular, it is defined by
    \begin{align*}
        &\mathrm{sensitivity}_{\cZ, \cF}(z) :=\\
        &\qquad\sup_{f_1, f_2 \in \cF} \frac{(f_1(z) - f_2(z))^2}{\min\{\sum_{(s, a) \in \cZ}(f_1(s, a) - f_2(s, a))^2, \lambda\}},
    \end{align*}
    where $\lambda$ is defined by $T(H+1)^2$ for the RL task with $r_h(s, a) \in [0, 1]$ \footnote{We ignore the clipping process making $\mathrm{sensitivity}_{\cZ, \cF}(z) \leftarrow \min\{\mathrm{sensitivity}_{\cZ, \cF}(z)\}$ for the clarity of demonstration}. The major difference between the generalized eluder dimension and sensitivity is that the generalized eluder dimension incorporates the variance $\sigma_s^2$ into the historical observation $\cZ$ to craft the heterogeneous variance in $\cZ$. 
\end{remark}

Since $D_{\cF_h}^2$ in Definition~\ref{def:ged} is not computationally efficient in some circumstances, we approximate it via an oracle $\overline D_{\cF_h}^2$, which is formally defined in Definition~\ref{def:bonus-oracle}.
\begin{definition}[Bonus oracle $\overline D_{\cF_h}^2$] \label{def:bonus-oracle}
The bonus oracle returns a computable function $\overline D_{\cF_h}^2(z; z_{[t],h}, \sigma_{[t],h})$, which computes the estimated uncertainty of a state-action pair $z = (s, a) \in \cS \times \cA$ with respect to historical data $z_{[t],h}$ and corresponding weights $\sigma_{[t],h}$. It satisfies  
\begin{align*} 
D_{\cF_h}(z; z_{[t],h}, \sigma_{[t],h}) & \le \overline D_{\cF_h}(z; z_{[t],h}, \sigma_{[t],h}) \\ 
& \le C \cdot D_{\cF_h}(z; z_{[t],h}, \sigma_{[t],h}),
\end{align*} 
where $C$ is a fixed constant. 
\end{definition}

The covering numbers of the value function class and the bonus function class are introduced in the following definition.
\begin{definition}[Covering numbers of function classes] \label{def:oracle}
   For any $\epsilon > 0$, we define the following covering numbers of involved function classes: 
   \begin{enumerate}[leftmargin=*,nosep]
       \item For each $h \in [H]$, there exists an $\epsilon$-cover $\cC(\cF_h, \epsilon) \subseteq \cF_h$ with size $|\cC(\cF_h, \epsilon)| \le N_{\cF_h} (\epsilon)$, such that for any $f \in \cF_h$, there exists $f' \in \cC(\cF_h, \epsilon)$ satisfying $\|f - f'\|_\infty \le \epsilon$. For any $\epsilon > 0$, we define the uniform covering number of $\cF$ with respect to $\epsilon$ as $N_\cF(\epsilon) := \max_{h \in [H]} N_{\cF_h}(\epsilon)$. 
       \item There exists a bonus function class $\cB=\{B : \cS \times \cA \to \RR\}$ such that for any $t \ge 0$, $z_{[t]} \in (\cS \times \cA)^t$, $\sigma_{[t]} \in \RR^t$, $h \in [H]$, the bonus function $\overline D_{\cF}(\cdot; z_{[t]}, \sigma_{[t]})$ returned by the bonus oracle in Definition \ref{def:bonus-oracle}  belongs to $\cB$. 
       \item For the bonus function class $\cB$, there exists an $\epsilon$-cover $\cC(\cB, \epsilon) \subseteq \cB$ with size $|\cC(\cB, \epsilon)| \le N_{\cB}(\epsilon)$, such that for any $b \in \cB$, there exists $b' \in \cC(\cB, \epsilon)$, such that $\|b - b'\|_\infty \le \epsilon$.
       \item The optimistic function class at stage $h \in [H]$ is:
\begin{small}
    \begin{align*}
            \cV_{h}  = \Big\{V(\cdot)  &= \max_{a \in \cA}  \min \Big(1, f(\cdot, a)  \\
            &\qquad + \beta \cdot b(\cdot, a)\Big) \bigg| f \in \cF_h, b \in \cB \Big\}.
        \end{align*}
\end{small}       
        There exists an $\epsilon$-cover $\cC(\cV_h, \epsilon)$ with size $ |\cC(\cV_h, \epsilon)| \le N_{\cV_h}(\epsilon) $.  For any $\epsilon > 0$, we define the uniform covering number of $\cV$ with respect to $\epsilon$ as $N_\cV(\epsilon) := \max_{h \in [H]} N_{\cV_h}(\epsilon)$. 
   \end{enumerate}
\end{definition}

\section{Algorithm}
In this section, we introduce our algorithm \alg~as presented in Algorithm~\ref{alg:main}. \alg~consists of two phases, where in the first exploration phase, \alg~collects $K$ episodes without reward signal. Then in the second planning phase, \alg~leverages the collected $K$ episodes to learn a policy trying to maximize the cumulative reward given a specific reward function $r$. The details of these two phases are presented in the following subsections.

\subsection{Exploration Phase: Efficient Exploration via Uncertainty-aware Intrinsic Reward}
The ultimate goals of the exploration phase are exploring environments and collecting data in the absence of reward to facilitate finding the near-optimal policy in the next phase. At a high level, \alg~achieves these goals by encouraging the agent to explore regions containing higher uncertainty, which intuitively guarantees the maximal information gained in each episode. 

\noindent\textbf{Intrinsic reward.} \alg~ evaluate the uncertainty by $D_{\cF_h}$ in Definition~\ref{def:ged}, and uses its oracle $\overline D_{\cF_h}$ as the intrinsic reward $r_{k,h} $ in Line~\ref{ln:intrinsic_reward} to generate an uncertainty-target policy in Line~\ref{ln:exploration_policy}. Recall that $D_{\cF_h}^2(z;  z_{[k - 1],h}, \sigma_{[k - 1],h})$ is defined as
\begin{align*}
        \sup_{f_1, f_2 \in \cF_h} \frac{(f_1(z) - f_2(z))^2}{\sum_{s \in [i - 1]} \frac{1}{\sigma_{s,h}^2}(f_1(z_{s,h}) - f_2(z_{s,h}))^2 + \lambda}.
\end{align*}
In particular, a high reward signal means that there exist functions in $\cF_h$ close to each other on all historical observations but divergent for the current state and action pair. This further suggests that the past observations are not enough for the agent to make a precise value estimation for the current state-action pair.

\noindent\textbf{Weighted regression.} The usage of the intrinsic reward $r_{k,h}$ induces an intrinsic action-value function $Q_{k,h}^*(\cdot, \cdot; r_k)$, which serves as a metric for cumulative uncertainty of remaining stages. As in model-free approaches, \alg~aims to estimate $Q_{k,h}^*(\cdot, \cdot; r_k)$ and further finds a policy $\pi_h^k$ that would maximize the cumulative uncertainty over $H$ stages. This part is presented in Algorithm~\ref{alg:main} through Line~\ref{ln:hat_f} to Line~\ref{ln:exploration_policy}.

To reduce the estimation error, \alg~ incorporates the weighted regression proposed in~\citet{zhao2023nearly} into estimating $Q_{k,h}^*(s,a; r_k)$. The algorithm starts at final stage $h = H$ and estimating the $Q_{k,h}^*(s,a; r_k)$ approximated by function $\hat f_{k,h}$ using Bellman equation:  
\begin{align*}
    \hat f_{k,h}(s_h, a_h) &= r_{k,h}(s_h, a_h) + [\PP_h V_{k,h+1}](s_h, a_h) \\
    &\approx r_{k,h}(s_h, a_h) + V_{k,h+1}(s_{h+1}).
\end{align*}
However, estimating $[\PP_h V_{k,h+1}](s_h, a_h)$ using $V_{k,h+1}(s_{h+1})$ may also introduce error since the variance of distribution $\PP_h (\cdot | s, a)$ varies among different state-action pair. Therefore, we tackle this heterogeneous variance issue by minimizing the Bellman residual loss weighted by using the estimated variance $\bar \sigma_{k,h}$ of observed state-action pairs $s_h^i, a_h^i$:
\begin{align*}
 \sum_{i \in [k-1]} \frac{(f_{i,h}(s_h^i, a_h^i) - r_{i,h}(s_h^i,a_h^i) - V_{i,h+1}(s_{h + 1}^i))^2}{\bar \sigma_{i, h}^2}.
\end{align*}
Obviously, a lower variance $\bar \sigma_{i, h}$ yields a larger weight during the regression. The calculation of variances $\bar \sigma_{i, h}$ involves both \textit{aleatoric uncertainty} and \textit{epistemic uncertainty} \citep{kendall2017uncertainties, mai2022sample}, where the \textit{aleatoric uncertainty} is $\sigma_{k,h}$ calculated in Line~\ref{ln:sigma} caused by indeterminism of the transition and \textit{epistemic uncertainty} is $\overline{D}^{1/2}_{\cF_h}$ caused by limited data. Such an approach can be proved to improve the sample efficiency of our algorithm \alg~ (see Theorem \ref{thm:main} and its discussion). Similar approaches have been used in~\citet{zhou2021nearly, ye2023corruption} to provide more robust and efficient estimation.

After obtaining the $\hat f_{k,h}$ function through weighted regression, \alg~follows the standard optimism design in online exploration methods to add the bonus term $b_{k,h}$ for overestimating the $Q_{k,h}^*(s,a; r)$ function in Line~\ref{ln:Q}. Using this optimistic estimation, \alg~thus takes the greedy policy and estimates the value function $V_{k,h}$ in Line~\ref{ln:V} before proceeding to the previous stage $h - 1$.


\subsection{Planning Phase: Effective Planning Using Weighted Regression}
After exploring environments and collecting data in the exploration phase, the agent is now given the reward for a specific task, but no longer interacts with the environment. \alg~enters its planning phase and ensures a policy to maximize the cumulative reward of $r_h$ across all $H$ stages. \alg~estimates $Q_h^*(s,a; r)$ by weighted regression and further finds the optimal policy $\pi_h$, which is the same process as in the exploration phase. This part is presented in Algorithm\ref{alg:main} through Line~\ref{ln:planning_bonus} to Line~\ref{ln:planniing_policy}.

\begin{remark}
    Compared with~\citet{kong2021online}, our algorithm leverages the advantage of generalized elude dimension and incorporates the estimated variance $\sigma$ into 1) weighted regression in Line~\ref{lnn:regression} in the planning phase and Line~\ref{ln:hat_f} in exploration phase; 2) intrinsic reward design in Line~\ref{ln:intrinsic_reward}. Also, our algorithm does not set the reward $r_{k, h} = b_{k, h} / H$ as of~\citet{kong2021online, wang2020reward}, thus the agent can explore more aggressively and more efficiently using the knowledge of variance of the observation. Therefore, \alg~is more sample efficient compared with~\citet{kong2021online}, which is discussed in detail in Remark~\ref{rm:kong}.
\end{remark}

\begin{algorithm}[!h]
\caption{\alg}\label{alg:main}
\label{alg:exp}
\begin{algorithmic}[1]
\REQUIRE Confidence radius $\beta^E$
\REQUIRE Regularization parameter $\lambda$
\STATE \textbf{Phase I: Exploration Phase}
\FOR {$k=1,2,\cdots,K$}
\FOR {$h=H,H-1,\cdots,1$}
\STATE $b_{k,h}(\cdot,\cdot) \gets 2 {\beta}^E \cdot \overline{\cD}_{\cF_h}(\cdot,\cdot;z_{[k-1],h},\bar{\sigma}_{[k-1],h}).$ 
\STATE $r_{k,h}(\cdot,\cdot) \gets b_{k,h}(\cdot.\cdot) / 2$. \label{ln:intrinsic_reward}
\STATE $\hat f_{k, h} \gets  \argmin_{f_h \in \cF_h} \sum_{i \in [k - 1]} \frac{1}{\bar \sigma_{i, h}^2} (f_h(s_h^i, a_h^i) - r_{k,h}(s_h^i,a_h^i) - V_{k,h+1}(s_{h + 1}^i))^2$. \label{ln:hat_f} 
\STATE ${Q}_{k,h}(s,a) \gets \min\Big\{\hat f_{k, h}(s, a)+b_{k, h}(s, a), 1\Big\}$. \label{ln:Q}
\STATE ${V}_{k,h}(s) \gets \max_{a} {Q}_{k,h}(s,a)$. \label{ln:V}
\STATE Set the policy $\pi_h^{k}(\cdot) \gets \argmax_{a \in \cA} Q_{k,h}(\cdot,a)$.  \label{ln:exploration_policy}
\ENDFOR
\STATE Receive the initial state $s_1^k$.
\FOR{stage $h=1,\ldots,H$}
\STATE Take action $a_h^k\leftarrow \pi_h^k(s_h^k)$, receive next state $s_{h+1}^k$.
\STATE $\textstyle{\sigma_{k,h} \leftarrow 2 \sqrt{\log N_\cV(\epsilon) \cdot \min\{\hat f_{k,h} (s_h^k,a_h^k),1  \}}}$.\label{ln:sigma}
\STATE $\bar\sigma_{k,h}\leftarrow \max\big\{\gamma \cdot \overline{D}^{1/2}_{\cF_h} (z_{k,h}; z_{[k - 1],h}, \bar\sigma_{[k - 1],h}), $ $\sigma_{k,h}, \alpha \big\}$. \label{ln:hat_sigma}
\label{algorithm:line2}
\ENDFOR
\ENDFOR 
\STATE \textbf{Phase II: Planning Phase}
\REQUIRE Dataset $\{(s_h^k, a_h^k, \bar \sigma_{k, h}^2)  \}_{(k,h)\in [K] \times [H]}$
\REQUIRE Confidence radius $\beta^P$ 
\REQUIRE Reward function $r=\{ r_h \}_{h \in [H]}$
\STATE Initiate $\hat{V}_{H+1}(\cdot) \gets 0$, $\hat{Q}_{H+1}(\cdot,\cdot) \gets 0 $
\FOR{step $h=H,\cdots,1$}
\STATE $b_h (\cdot,\cdot) \gets \min \{ {\beta}^P  \overline{\cD}_{\cF_h}(z;z_{[K],h},\bar{\sigma}_{[K],h}), 1\}$. \label{ln:planning_bonus}
\STATE $\hat{f}_{h} \gets \argmin_{f_h \in \cF_h} \sum_{i \in [K]} \frac{1}{\bar \sigma_{i, h}^2} (f_h(s_h^i, a_h^i) - r_{h}(s_h^i,a_h^i) - \hat{V}_{h+1}(s_{h + 1}^i))^2$.\label{lnn:regression}
\STATE $\hat{Q}_{h}(s,a) \gets \min\Big\{\hat f_{h}(s, a)+b_{ h}(s, a) ,1\Big\}$. \label{ln:hat_Q} 
\STATE $\hat{V}_h(\cdot) \leftarrow \max_{a\in\cA} \hat{Q}_h(\cdot,a)$. \label{ln:hat_V}
\STATE $\pi_h (\cdot) \leftarrow \argmax_{a\in\cA} \hat{Q}_h(\cdot,a)$. \label{ln:planniing_policy}
\ENDFOR
\ENSURE Policy $\pi$
\end{algorithmic}
\end{algorithm}

\section{Theoretical Results}\label{sec:results}
We analyze \alg\ theoretically in this section. The uncertainty-aware reward-free exploration mechanism leads to efficient learning with provable sample complexity guarantees. The first theorem characterizes how the sub-optimality decays as exploration time grows.
\begin{theorem} \label{thm:main}
    For \alg, set confidence radius $\beta^E = \tilde{O} \big(\sqrt{H \log N_{\cV}(\epsilon)}\big)$ and $\beta^P = \tilde{O} \big(\sqrt{H \log N_{\cF}(\epsilon)}\big)$, and take $\alpha = 1 / \sqrt{H}$ and $\gamma = \sqrt{\log N_{\cV}(\epsilon)}$. Then, for any $\delta \in (0,1)$, with probability at least $1-\delta$, after collecting $K$ episodes of samples, for any reward function $r=\{ r_h \}_{h=1}^H$ such that $\sum_{h=1}^H r_h(s_h, a_h) \le 1$, \alg~outputs a policy $\pi$ satisfying the following sub-optimality bound,
    \begin{align*}
     &\EE_{s_1 \sim \mu} [ V_1^*(s_1;r) - V_1^{\pi}(s_1;r)]\\
     &= \tilde{O} \Big( H \sqrt{ \log N_{\cF} (\epsilon) } \sqrt{\dim_{\alpha,K} (\cF) / K}  \Big).
    \end{align*}
\end{theorem}
We are now ready to present the sample complexity of \alg~for the reward-free exploration.
\begin{corollary}\label{col:main}
    Under the same conditions in Theorem~\ref{thm:main}, with probability at least $1-\delta$, for any reward function $r=\{ r_h \}_{h=1}^H$ such that $\sum_{h=1}^H r_h(s_h, a_h) \le 1$, \alg~returns an $\epsilon$-optimal policy after collecting $K \le \tilde \cO(H^2 \log N_{\cF}(\epsilon) \dim_{\alpha,K}(\cF)\epsilon^{-2})$ episodes during the exploration phase.
\end{corollary}
\begin{remark}
Let $d_{K, \delta}$ be $\max\{\log N_{\cF}(\epsilon), \dim_{\alpha,K}(\cF)\}$, \alg~yields an $\tilde \cO(H^2 d^2_{K, \delta}\epsilon^{-2})$ sample complexity for reward-free exploration with high probability. In tabular setting, $d_{K, \delta} = \tilde \cO(SA)$, thus yields an $\tilde \cO(H^2S^2A^2\epsilon^{-2})$  sample complexity. In linear MDPs and generalized linear MDPs with dimension $d$, $d_{K, \delta} = \tilde \cO(d)$, thus yields an $\tilde \cO(H^2d^2\epsilon^{-2})$ sample complexity which matches the result from~\citet{hu2022towards}.
For a more general setting where the function class with eluder dimension $d$, $d_{K, \delta} = \tilde \cO(d)$, which yields a $\tilde \cO(H^2 d^2 \epsilon^{-2})$ sample complexity.
\end{remark}
For a fair comparison with some existing works, we translate our sample complexity result to the case where the reward scale is $r_h \in [0,1]$, $\forall h \in [H]$. The result can be trivially obtained by replacing $r \rightarrow r / H$ in \alg. 
\begin{corollary}
    With probability at least $1 - \delta$, for any reward function such that $r_h(s, a) \in [0, 1]$ or the total reward is bounded by $\sum_{h=1}^H r_h(s_h, a_h) \le H$,  \alg~returns an $\epsilon$-optimal policy after collecting $K \le \tilde \cO(H^4 d_{K, \delta}^2\epsilon^{-2})$ episodes in the exploration phase. 
\end{corollary}
\begin{remark}
Compared with~\citet{chen2022unified} which provides a $\tilde \cO(d\log |\cP| \epsilon^{-2})$ sample complexity for model-based RL, \alg~is a \emph{model-free} algorithm which does not need to directly sample transition kernel $\PP_h(\cdot | \cdot, \cdot)$ from all possible transitions $\tilde \Delta(\Pi)$, therefore, \alg~is computationally efficient and can be easily implemented based on the current empirical DRL algorithms. 
\end{remark}
\begin{remark}
    Compared with~\citet{chen2022statistical} which achieves a $\tilde \cO(H^7d^3\epsilon^{-2})$ sample complexity, one can find our result significantly improves the dependency on $H, d$. \citet{chen2022statistical} didn't optimize the exploration policy by constructing intrinsic rewards but by updating Bellman error constraints on the value function class. It sacrificed the sample complexity to adapt the general function approximation settings. In addition, this approach is generally computationally intractable as it explicitly maintains feasible function classes. For its V-type variant, it even maintains a finite cover of the function class, which can be exponentially large.
\end{remark}
\begin{remark}\label{rm:kong}
    \citet{kong2021online} leveraged the ``sensitivity" as the intrinsic reward during the exploration and achieved a $\tilde \cO(H^6d^4\epsilon^{-2})$ reward-free sample complexity. Compare their algorithm and ours, ours improves a $H^2d^2$ factor from 1) using weighted regression to handle heterogeneous observations 2) using a ``truncated Bellman equation''~\citep{chen2022nearoptimal} in our analysis, and 3) a properly improved uncertainty metric $\overline D^2_{\cF_h}$ instead of the sensitivity. 
\end{remark}

\section{Experiments}
\begin{table*}[!t]
\centering
\caption{Cumulative reward for various exploration algorithms across different environments and tasks. The cumulative reward is averaged over $8$ individual runs for both online exploration and offline planning. The result for each individual run is obtained by evaluating the policy network using the last-iteration parameter. Standard deviation is calculated across these runs. Results presented in \textbf{boldface} denote the best performance for each task, and those \underline{underlined} represent the second-best outcomes. The cyan background highlights results of our algorithms.}
\resizebox{\textwidth}{!}{%
\begin{tabular}{c|c|c|c|c|c|c|c|c|g}
\toprule
\multirow{2}{*}{Environment} & \multirow{2}{*}{Task}  &  \multicolumn{7}{c|}{Baselines}  & \cellcolor{white} Ours \\
\cmidrule{3-9}
 & & ICM & APT &  DIAYN  & APS & Dis. & SMM & RND   & \cellcolor{white} \alg \\ 
\midrule
\multirow{4}{*}{Walker} & Flip & 177 $\pm$ 80 & 523 $\pm$ 57 & 207 $\pm$ 119 & 246 $\pm$ 103 & \textbf{570 $\pm$ 32} &242 $\pm$ 71 & 507 $\pm$ 48 & \underline{554 $\pm$ 64}\\
& Run & 108 $\pm$ 41 & 304 $\pm$ 38 & 113 $\pm$ 38 & 132 $\pm$ 39 & \textbf{340 $\pm$ 37} & 116 $\pm$ 21 & 306 $\pm$ 34   & \underline{339 $\pm$ 34}\\
& Stand & 466 $\pm$ 17& \underline{ 891 $\pm$ 62} & 587 $\pm$ 169 & 573 $\pm$ 177& 726 $\pm$ 79& 443 $\pm$ 104 & 750 $\pm$ 62 & \textbf{925 $\pm$ 50}\\
& Walk & 411$\pm$237 & 772$\pm$60  & 432 $\pm$ 222 & 645 $\pm$ 156 & \textbf{851 $\pm$ 63} &  273 $\pm$ 162 & 709 $\pm$ 115  & \underline{826 $\pm$ 89}\\
\midrule
\multirow{4}{*}{Quadruped} & Run & 93 $\pm$ 68 & 452 $\pm$ 49 & 158 $\pm$ 64 & 159 $\pm$ 82 & \textbf{524 $\pm$ 24} & 162 $\pm$ 140 & \underline{522 $\pm $ 30}   & \Gape[0pt][2pt] 460 $\pm$ 36\\
& Jump & 89 $\pm$ 47 & 740 $\pm$ 91 & 218 $\pm$ 114 & 123 $\pm$67 & \textbf{829 $\pm$ 22} &  211 $\pm$ 127 & \underline{790 $\pm$ 38} & 719 $\pm$ 68 \\
& Stand & 207 $\pm$ 134 & 910 $\pm$ 45 & 331 $\pm$ 81 & 308 $\pm$ 147 & \textbf{953 $\pm$ 16} & 239 $\pm$ 104 & \underline{940 $\pm$ 27}  & 867 $\pm$ 61\\
& Walk & 94 $\pm$ 60& 680 $\pm$ 117 & 171 $\pm$ 72 & 141 $\pm$ 80 & 720 $\pm$ 175& 125 $\pm$ 36 & \textbf{820 $\pm$ 94} & \underline{726 $\pm$ 146}\\
\bottomrule
\end{tabular}%
}
\vspace{-1em}
\label{tab:2}
\end{table*}
\subsection{Experiment Setup}
Based on our theoretical perceptive, we integrate our algorithm in the unsupervised reinforcement learning (URL) framework and evaluate the performance of the proposed algorithm in URL benchmark~\citep{laskin2021urlb}.\footnote{Our implementation can be accessed at GitHub via \url{https://github.com/uclaml/GFA-RFE}.} As suggested by~\citet{ye2023corruption}, we use the variance of $n$-ensembled $Q$ functions as the estimation of the bonus oracle $\overline D_{\cF}^2$ which will be used in (1) intrinsic reward $r_{k, h}$; (2) exploration bonus $b_{k,h}$; and (3) weights $\sigma_{k, h}^2$ for the value target regression. All these $Q$ networks are trained by Q-learning with different mini-batches in the replay buffer. Obviously, the variance of these $Q$ networks comes from the randomness of initialization and the randomness of different mini-batches used in training. The pseudo code for the practical algorithm is deferred to Appendix~\ref{app:experiment_details}.

The original implementation of~\citet{laskin2021urlb} involves two phases where the neural network is first \emph{pretrained} by interacting with the environment without receiving reward signals and then \emph{finetuned} by interacting with the environment again with reward signal. However, in our experiments, we strictly follow the design of reward-free exploration by first exploring the environment without the reward. The explored trajectories are collected into a dataset $\cD = \{(s, a, s')\}$. Then we call a reward oracle $r$ to assign rewards to this dataset $\cD$ and learn the optimal policy using the offline dataset $\cD_r = \{(s, a, s', r(s, a, s'))\}$ without interacting the dataset anymore. Intuitively speaking, this \emph{online exploration + offline planning} paradigm is more challenging than the \emph{online pretraining + online finetuning} and would be more practical, especially with different reward signals.

\noindent\textbf{Unsupervised reinforcement learning benchmarks.}
We conduct our experiments on \emph{Unsupervised Reinforcement Learning Benchmarks}~\citep{laskin2021urlb}, which consists of two multi-tasks environments (\emph{Walker}, \emph{Quadruped}) from DeepMind Control Suite~\citep{tunyasuvunakool2020}. Each environment is equipped with several reward functions and goals. For example, \emph{Walker-run} consists of rewards encouraging the walker to run at speed and \emph{Walker-stand} consists of rewards indicating the walker should stead steadily. We consider the state-based input in our experiments where the agent can directly observe the current state instead of image inputs (a.k.a. pixel-based).

\noindent\textbf{Baseline algorithms.} We inherit the baseline algorithms ICM~\citep{pathak2017curiosity}, Disagreement~\citep{pathak2019self}, RND~\citep{burda2018exploration}, APT~\citep{liu2021behavior}, DIAYN~\citep{eysenbach2018diversity}, APS~\citep{liu2021aps}, SMM~\citep{lee2019efficient}. All these algorithms provide different `intrinsic rewards' in place of ours during exploration. We make all these baseline algorithms align with our settings which first collect an exploration dataset and then do offline training on the collected dataset with rewards. 

\subsection{Experiment Results}
Experimental results are presented in Table~\ref{tab:2}. It's obvious that \alg~can efficiently explore the environment without the reward function and then output a near-optimal policy given various reward functions. For the baseline algorithms, APT, Disagreement, and RND perform consistently better than the rest of the 4 algorithms on all 2 environments and 8 tasks. The performance of \alg~enjoys compatible or superior performance compared with these top-level methods (APT, Disagreement, and RND), on these tasks. These promising numerical results justify our theoretical results and show that \alg~can indeed efficiently learn the environment in a practical setting.

\noindent \textbf{Ablation study.} To verify the performance of our algorithm, we also did ablation studies on 1) the relationship between offline training processes and episodic reward 2) the relationship between the quantity of online exploration data used in offline training and the achieve episodic reward. The details of the ablation study are deferred to Appendix~\ref{app:experiment_details}.
\section{Conclusion}
We study reward-free exploration under general function approximation in this paper. We show that, with an uncertainty-aware intrinsic reward and variance-weighted regression on learning the environment, \alg~can be theoretically proved to explore the environment efficiently without the existence of reward signals. Experiments show that our design of intrinsic reward can be efficiently implemented and effectively used in an unsupervised reinforcement learning paradigm. In addition, experiment results verify that adding uncertainty estimation to the learning processes can improve the sample efficiency of the algorithm, which is aligned with our theoretical results of weighted regression. 
\section*{Acknowledgements}

We thank the anonymous reviewers for their helpful comments. WZ is supported by the UCLA Dissertation Year Fellowship. QG is supported in part by the National Science Foundation CAREER Award 1906169 and research fund from UCLA-Amazon Science Hub. The views and conclusions contained in this paper are those of the authors and should not be interpreted as representing any funding agencies.
\section*{Impact Statement}

This paper discusses efforts aimed at enhancing the domain of reinforcement learning, encompassing both theoretical insights and empirical findings. Our approaches are poised to offer beneficial societal effects as they circumvent the reliance on pre-defined reward functions, which could embed biases and prejudicial assessments. Furthermore, our robust theoretical foundation enhances the transparency  of exploration processes within reinforcement learning frameworks. While our algorithm potentially improve the accountability of exploration of RL, it does not eliminate the risks of unsafe behavior or offending to the constrain of exploration, which might be addressed in the future works. Other than that we do not see any potential negative impacts of our work that need to be specifically pointed out here, to the best of our knowledge.

\bibliography{main}
\bibliographystyle{icml2024}

\newpage
\appendix
\onecolumn

\section{Proof of Theorems in Section~\ref{sec:results}}\label{app:level0}

\subsection{Additional Definitions and High Probability Events}

In this section, we introduce additional definitions that will be used in the proofs. Also, we define the good events that \alg\ is guaranteed to have near-optimal sample complexity.
\begin{definition}[Truncated Optimal Value Function] \label{def:truncated_optimal} We define the following truncated value functions for any reward $r$:
    \begin{align*}
        \tilde{V}^*_{H+1}(s;r) & = 0, \quad \forall s \in \cS \\
        \tilde{Q}^*_h(s,a;r) & = \min \{  r_h(s,a) + \PP_h \tilde{V}^*_{h+1} (s,a;r), 1 \}, \quad \forall (s,a) \in \cS \times \cA \\
        \tilde{V}^*_{h}(s;r) & = \max_{a\in\cA} \tilde{Q}^*_h(s,a;r). \quad \forall s \in \cS, h\in [H].
    \end{align*}
\end{definition}

The good event $\cE_{k, h}^{E}$ at stage $h$ of episode $k$ in exploration phase is defined to be:
\begin{align*}
 \cE_{k, h}^{E} = \Big\{\lambda + \sum_{i \in [k - 1]} \frac{1}{\bar \sigma_{i, h}^2}\left(\hat f_{k, h}(s_h^i, a_h^i) - \cT_h V_{k, h + 1}(s_h^i, a_h^i)\right)^2 \le  ({\beta}^E)^2\Big\}.
\end{align*}
The intersection of all good events in exploration phase is:
\begin{align*}
   \cE^{E}:= \bigcap_{k \ge 1, h \in [H]}  \cE_{k, h}^{E}.
\end{align*}
The following lemma indicates that $\cE$ holds with high probability for \alg.
\begin{lemma}\label{lm:event-exp}
In Algorithm~\ref{alg:main}, for any $\delta\in(0,1)$ and fixed $h\in[H]$, with probability at least $1-\delta$, $\cE^{E}$ holds.
\end{lemma}
In the planning phase, we define the good events for exploration phase with indicator functions as
\begin{align*}
 \overline{\cE}_{h}^{P} &= \Big\{\lambda + \sum_{i \in [K]} \frac{\ind_{h}}{\bar \sigma_{i, h}^2}\left(\hat f_{h}(s_h^i, a_h^i) - {\cT}_h \hat{V}_{h + 1}(s_h^i, a_h^i)\right)^2 \le  (\hat{\beta}^{P})^2\Big\}, \\
 \overline{\cE}^{P} & = \bigcap_{h \in [H]}\overline{\cE}_{h}^{P},
\end{align*}
where $\hat{\ind}_h = \ind (V^*_{h+1}(s) \le \hat V_{h+1}(s),~\forall s \in \cS) \cdot \ind (\hat V_{h+1}(s) \le V_{k,h+1}(s)+V^*(s;r),~\forall s \in \cS) \cdot \ind([\VV_h (\hat V_{h+1} - V^*_{h+1})](s_h^k,a_h^k) \le \eta^{-1} \bar \sigma_{k,h}^2,~\forall k \in  [K] )$ and $\eta =  \log N_{\cV}(\epsilon)$.
Like in the exploration phase, we also have that $\overline{\cE}^{P}$ holds with high probability for \alg.
\begin{lemma}\label{lm:overline-event-plan}
In Algorithm~\ref{alg:main}, for any $\delta\in(0,1)$ and fixed $h\in[H]$, with probability at least $1-\delta$, $\overline{\cE}^{P}$ holds.
\end{lemma}
Furthermore, we have the following good events in the planning phase without indicator function:
\begin{align*}
    \cE_{h}^{P} =\Big\{\lambda + \sum_{i=1}^{k-1} \frac{1}{(\bar \sigma_{i, h'})^2}\left(\hat f_{h'}(s_{h'}^i, a_{h'}^i) - \cT_{h'} V_{h' + 1}(s_{h'}^i, a_{h'}^i)\right)^2 \le (\beta^P)^2,\ \forall h\leq h'\leq H, k\in[K]\Big\}.
\end{align*}
And we define $\cE^{P} := \cE_{1}^{P}$. We shows that $\cE^P$ holds if both $\cE^E, \overline{\cE}^P$ hold with the help of the following lemma:
\begin{lemma}\label{lm:event-plan}
    If the event $\cE^E, \overline{\cE}^P, \cE_{h+1}^P$ all hold, then event $\cE^P_h $ holds.
\end{lemma}
Since $\cE_{H}^P$ holds trivially, Lemma~\ref{lm:event-plan} indicates that $\cE^{P}$ holds.
\subsection{Covering Number}

The optimistic value functions at stage $h \in [H]$ in our construction belong to the following function class: 
\begin{align} 
   \cV_{h} = \left\{V(\cdot)  = \max_{a \in \cA}  \min \left(1, f(\cdot, a) + \beta \cdot b(\cdot, a)\right) \bigg| f \in \cF_h, b \in \cB \right\}.\label{def:opt_value_class}
\end{align}

\begin{lemma}[$\epsilon$-covering number of optimistic value function classes] \label{lemma:opt-cover}
For optimistic value function class $\cV_{k, h}$ defined in \eqref{def:opt_value_class}, we define the distance between two value functions $V_1$ and $V_2$ as $\|V_1 - V_2\|_\infty := \max_{s \in \cS} |V_1(s) - V_2(s)|$. Then the $\epsilon$-covering number with respect to the distance function can be upper bounded by 
\begin{align} 
   N_{\cV_h}(\epsilon):=  N_{\cF_h}(\epsilon / 2) \cdot N_\cB(\epsilon / 2\beta) .\label{eq:v-cover}
\end{align}
\end{lemma}
Lemma~\ref{lemma:opt-cover} further indicates that
\begin{align*}
    N_{\cV}(\epsilon) = \max_{h\in [H]} N_{\cV_h}(\epsilon) = \max_{h\in [H]} N_{\cF_h}(\epsilon / 2) \cdot N_\cB(\epsilon / 2\beta) = N_{\cF}(\epsilon / 2) \cdot N_\cB(\epsilon / 2\beta).
\end{align*}

\subsection{Proof of Theorems}
We first introduce the following lemmas to build the path to Theorem~\ref{thm:main}.

\begin{lemma}\label{lm:concentration-plan}
    On the event $\cE^{P}$, we have
\begin{align*}
        |\hat{f}_{h}(s,a) - \cT_h \hat{V}_{h+1}  | \le {\beta}^P D_{\cF_{h}}(z;z_{[K],h},\bar{\sigma}_{[K],h}).
\end{align*}   
\end{lemma}

\begin{lemma}[Optimism in the planning phase]\label{lm:optimism-plan}
    On the event $\cE^{P}$, for any  $h\in[H]$,  we have
    \begin{align*}
        V_h^*(s;r) \le \hat{V}_{h}(s),\quad \forall s \in \cS.
    \end{align*}
\end{lemma}

\begin{lemma}\label{lm:value-sum}
    On the event $\underline{\cE}^E$, with probability at least $1-3\delta$, we have 
    \begin{align*}
        \sum_{k=1}^K V_{k,1} (s_1^k) = O(\beta^E \sqrt{\dim_{\alpha, K}(\cF)} H \sqrt{K}).
    \end{align*}
\end{lemma}

\begin{lemma}\label{lm:deviation-sum}
    With probability $1-\delta$, we have
    \begin{align*}
        \Big| \sum_{k=1}^K \big(  \EE_{s\sim\mu} \big[ \tilde{V}_1^*(s;r_k) \big]  - \tilde{V}_1^*(s;r_k) \big)  \Big| \le \sqrt{2K \log(1/\delta)}.
    \end{align*}
\end{lemma}

We denote the event that Lemma~\ref{lm:value-sum} holds as $\Phi$, and the event that Lemma~\ref{lm:deviation-sum} holds as $\Psi$.

\begin{lemma}\label{lm:value_bonus}
    Under event $\cE^E \cap \Phi \cap \Psi$, we have
    \begin{align*}
        \EE_{s\sim\mu} \Big[ \tilde{V}_1^* (s;b)\Big] = O\Big(\beta^E \sqrt{H \dim_{\alpha,K} (\cF) / K} \sqrt{ \log N_\cF (\epsilon) / \log N_{\cV}(\epsilon) } \Big),
    \end{align*}
    where $b = \{ b_h \}_{h=1}^H $ is the UCB bonus in planning phase.
\end{lemma}
With these lemmas, we are ready to prove Theorem~\ref{thm:main}.
\begin{proof}[Proof of Theorem~\ref{thm:main}]
By Lemma~\ref{lm:optimism-plan}, we can upper bound the suboptimality as
\begin{align*}
    \EE_{s_1 \sim \mu} [ V_1^*(s_1;r) - V_1^\pi(s_1;r)] 
\le \EE_{s_1 \sim \mu} [ \hat{V}_1(s_1) - V_1^\pi(s_1;r)].
\end{align*}

Then, we can decompose the difference between optimistic estimate of value function and the true value function in the following: 
\begin{align*}
 &\EE_{s_1 \sim \mu} [ \hat{V}_1(s_1) - V_1^\pi(s_1;r)] \\ 
 & =   \EE_{s_1 \sim \mu} \Big[ \min \{ \hat{f}_{1}(s_1,\pi(s_1)) + b_1(s_1,\pi(s_1)),1\}  - r_1(s_1,\pi(s_1)) - \PP_1V_2^\pi(s_1,\pi(s_1);r)\Big] \\
 & \le  \EE_{s_1 \sim \mu} \Big[ \min \{ \hat{f}_{1}(s_1,\pi(s_1)) + b_1(s_1,\pi(s_1)) - r_1(s_1,\pi(s_1)) - \PP_1V_2^\pi(s_1,\pi(s_1);r),1\}  \Big] \\
 &=  \EE_{s_1 \sim \mu} \Big[ \min \Big\{ \hat{f}_{1}(s_1,\pi(s_1))  - r_1(s_1,\pi(s_1))  - \PP_1\hat{V}_2^\pi(s_1,\pi(s_1);r) \\ 
 & \quad + \PP_1\hat{V}_2^\pi(s_1,\pi(s_1);r) + b_1(s_1,\pi(s_1)) - \PP_1V_2^\pi(s_1,\pi(s_1);r),1 \Big\} \Big] \\
 & =   \EE_{s_1 \sim \mu} \Big[ \min \Big\{ \hat{f}_{1}(s_1,\pi(s_1)) - \cT_1\hat{V}^{\pi}_{2}(s_1,\pi(s_1))+ \PP_1\hat{V}_2^\pi(s_1,\pi(s_1);r) \\
 & \qquad + b_1(s_1,\pi(s_1)) - \PP_1V_2^\pi(s_1,\pi(s_1);r),1  \Big\} \Big] \\
 & \le   \EE_{s_1 \sim \mu} \Big[ \min \Big\{ 2 b_1(s_1,\pi(s_1))  + \PP_1\hat{V}_2^\pi(s_1,\pi(s_1);r)  - \PP_1V_2^\pi(s_1,\pi(s_1);r),1 \Big\} \Big],
\end{align*}
where the last inequality holds due to Lemma~\ref{lm:concentration-plan}. Then, by the induction, we have
\begin{align*}
&\EE_{s_1 \sim \mu} [ \hat{V}_1(s_1) - V_1^\pi(s_1;r)] \\ 
& \le \EE_{s_1 \sim \mu} \Big[ \min \Big\{ 2 b_1(s_1,\pi(s_1))  + \PP_1\hat{V}_2^\pi(s_1,\pi(s_1);r)  - \PP_1V_2^\pi(s_1,\pi(s_1);r),1 \Big\} \Big] \\
& =  \EE_{s_1 \sim \mu, s_2\sim \PP (\cdot |s_1,\pi(s_1)) } \Big[ \min \Big\{ 2 b_1(s_1,\pi(s_1))  + \hat{V}_2^\pi(s_2;r)  - V_2^\pi(s_2;r),1 \Big\}   \Big] \\
& \le  \EE_{\tau \sim d^\pi} \Big[ \min \Big\{  \sum_{h=1}^H 2b_h(s_h,\pi(s_h)), 1   \Big\}\Big] \\ 
& \le  2 \EE_{s_1 \sim \mu} \Big[ \tilde{V}_1^\pi(s_1;b) \Big] \\
& \le  2 \EE_{s_1 \sim \mu} \Big[ \tilde{V}_1^*(s_1;b) \Big] \\
& =  O\Big(\beta^E \sqrt{H \dim_{\alpha,K} (\cF) / K} \sqrt{ \log N_{\cF} (\epsilon) / \log N_{\cV}(\epsilon)} \Big).
\end{align*} 
Therefore, by substituting $\beta^E = \tilde{O}\big( \sqrt{H\log N_{\cV}(\epsilon)}\big) $, we complete the proof:
\begin{align*}
     \EE_{s_1 \sim \mu} [ V_1^*(s_1;r) - V_1^\pi(s_1;r)] 
     = O\Big( H \sqrt{\dim_{\alpha,K} (\cF) / K} \sqrt{ \log N_{\cF} (\epsilon) } \Big) .
\end{align*}
\end{proof}

\begin{proof}[Proof of Corollary~\ref{col:main}]
    By solving $\EE_{s_1 \sim \mu} [ V_1^*(s_1;r) - V_1^\pi(s_1;r)] \le \epsilon$, we have that
    \begin{align*}
        K \ge \frac{H^2 \log N_{\cF} (\epsilon) \dim_{\alpha,K} (\cF)}{\epsilon^2}.
    \end{align*}
\end{proof}

\section{Proof of Lemmas in Appendix~\ref{app:level0}}\label{app:level1}

In this section, we prove the lemmas used in Appendix~\ref{app:level0}.
\begin{proof}[Proof of Lemma~\ref{lm:event-exp}]
We first prove that $ \cE_{k, h}^{E}$ holds with probability $1-\delta/(KH) $. We have $\cT_h V_{k,h+1} \in \cF_h$ due to Assumption~\ref{assumption:complete}. For any function $V:~S\rightarrow[0,1]$, let $\eta_h^k(V) = r_{k,h}(s_h^k,a_h^k) + V(s_{h+1}^k) - \cT_h V(s_h^k,a_h^k)$. For all $f\in\cF_h$, since $a^2-2ab = (a-b)^2 - b^2$, we have
\begin{align*}
    &\sum_{i\in[k-1]}\frac{1}{(\bar{\sigma}_{i,h})^2} \Big( f(s_h^i,a_h^i) - \cT_hV_{k,h+1}(s_h^i,a_h^i) \Big)^2 - 2\underbrace{\sum_{i\in[k-1]}\frac{1}{(\bar{\sigma}_{i,h})^2} \Big( f(s_h^i,a_h^i) - \cT_hV_{k,h+1}(s_h^i,a_h^i)\big) \eta_h^k(V_{k,h+1})}_{I(f,\cT_{h}V_{k,h+1},V_{k,h+1})} \\
    &= \sum_{i\in[k-1]}\frac{1}{(\bar{\sigma}_{i,h})^2}\Big( r_{k,h}(s_h^i,a_h^i) + V_{k,h+1}(s_{h+1}^i) - f(s_h^i,a_h^i) \Big)^2 - \sum_{i\in[k-1]}\frac{1}{(\bar{\sigma}_{i,h})^2} \eta_h^k(V_{k,h+1})^2.
\end{align*}
Take $f = \hat{f}_{k,h}$. By the the definition of $\hat{f}_{k,h}$, we have
\begin{align*}
    \sum_{i\in[k-1]}\frac{1}{(\bar{\sigma}_{i,h})^2} \Big( \hat{f}_{k,h}(s_h^i,a_h^i) - \cT_hV_{k,h+1}(s_h^i,a_h^i) \Big)^2 - 2 I(\hat{f}_{k,h},\cT_{h}V_{k,h+1},V_{k,h+1}) \le 0
\end{align*}
Applying Lemma~\ref{lemma:hoeffding-variant}, for fixed $f$, $\bar f$, and $V$, with probability at least $1-\delta$,
\begin{align*}
    I(f,\bar{f},V) &:= \sum_{i\in[k-1]}\frac{1}{(\bar{\sigma}_{i,h})^2} \Big( f(s_h^i,a_h^i) - \bar f (s_h^i,a_h^i)\Big) \eta_h^k(V) \\
    & \le  \frac{2 \tau}{\alpha^2} \sum_{i\in[k-1]}\frac{1}{(\bar{\sigma}_{i,h})^2} \Big( f(s_h^i,a_h^i) - \bar f (s_h^i,a_h^i)\Big)^2 + \frac{1}{\tau} \cdot \log \frac{1}{\delta}.
\end{align*}
Applying a union bound and take $\tau = \frac{\alpha^2}{8}$, for any $k$, with probability at least $1-\delta$, we have for all $V^c$ in the $\epsilon$-net $\cV_{h+1}$ that
\begin{align*}
    I(f,\bar{f},V^c) \le \frac{1}{4}\sum_{i\in[k-1]} \frac{1}{(\bar{\sigma}_{i,h})^2} \Big( f(s_h^i,a_h^i) - \bar f (s_h^i,a_h^i)\Big)^2 + \frac{2}{\alpha^2} \cdot \log\frac{N_\cV (\epsilon)}{\delta}
\end{align*}
For all $V$ such that $\|V - V^c\|_\infty \le  \epsilon$, we have $|\eta_{h}^i(V) - \eta_h^i(V^c)| \le 4\epsilon$. Thus,
\begin{align*}
    I(f,\bar{f},V_{k,h+1}) \le \frac{1}{4}\sum_{i\in[k-1]} \frac{1}{(\bar{\sigma}_{i,h})^2}\Big( f(s_h^i,a_h^i) - \bar f (s_h^i,a_h^i)\Big)^2 + \frac{2}{\alpha^2} \cdot \log\frac{N_\cV(\epsilon)}{\delta} + 4\epsilon \cdot kL / \alpha^2
\end{align*}
Applying a union bound, for any $k$, with probability at least $1-\delta$, we have for all $f^a,f^b$ in the $\epsilon$-net $\cC(\cF_h, \epsilon) $ that
\begin{align*}
    I(f^a,f^b,V_{k,h+1}) \le \frac{1}{4}\sum_{i\in[k-1]} \frac{1}{(\bar{\sigma}_{i,h})^2} \Big( f^a(s_h^i,a_h^i) - f^b (s_h^i,a_h^i)\Big)^2 + \frac{2}{\alpha^2} \cdot \log\frac{N_\cV (\epsilon)\cdot N_{\cF}(\epsilon)^2}{\delta}+ 4\epsilon \cdot kL / \alpha^2.
\end{align*}
Therefore, with probability at least $1-\delta$, we have
\begin{align*}
    & I(\hat{f}_{k,h},\cT_{h}V_{k,h+1},V_{k,h+1}) \le I(f^a,f^b,V_{k,h+1}) + 8\epsilon \cdot k / \alpha^2 \\
    & \le \frac{1}{4}\sum_{i\in[k-1]} \frac{1}{(\bar{\sigma}_{i,h})^2} \Big( f^a(s_h^i,a_h^i) - f^b (s_h^i,a_h^i)\Big)^2 + \frac{4}{\alpha^2} \cdot \log\frac{N_\cV (\epsilon)\cdot N_\cF (\epsilon)}{\delta}+ 4\epsilon \cdot kL / \alpha^2 + 8\epsilon \cdot k / \alpha^2 \\
    & \le \frac{1}{4}\sum_{i\in[k-1]} \frac{1}{(\bar{\sigma}_{i,h})^2} \Big( \hat{f}_{k,h}(s_h^i,a_h^i) - \cT_h V_{k,h+1} (s_h^i,a_h^i)\Big)^2 + \frac{4}{\alpha^2} \cdot \log\frac{N_\cV (\epsilon)\cdot N_\cF (\epsilon)}{\delta}+ 4\epsilon \cdot kL / \alpha^2 \\
    & \qquad + 8\epsilon \cdot k / \alpha^2 + 2L\epsilon\cdot k / \alpha^2 \\
    & \le \frac{1}{4}\sum_{i\in[k-1]} \frac{1}{(\bar{\sigma}_{i,h})^2} \Big( \hat{f}_{k,h}(s_h^i,a_h^i) - \cT_h V_{k,h+1} (s_h^i,a_h^i)\Big)^2 + \frac{4}{\alpha^2} \cdot \log\frac{N_\cV (\epsilon)\cdot N_\cF (\epsilon)}{\delta} + 14L\epsilon\cdot k / \alpha^2.
\end{align*}
Substituting it back, with probability at least $1-\delta/(KH)$, we have
\begin{align*}
    \frac{1}{4}\sum_{i\in[k-1]} \frac{1}{(\bar{\sigma}_{i,h})^2} \Big( \hat{f}_{k,h}(s_h^i,a_h^i) - \cT_h V_{k,h+1} (s_h^i,a_h^i)\Big)^2 \le \frac{16}{\alpha^2} \cdot \log\frac{KH \cdot N_\cV (\epsilon)\cdot N_\cF (\epsilon)}{\delta} + 56L\epsilon\cdot k / \alpha^2
\end{align*}
Take $ \alpha = 1 / \sqrt{H}$ and let 
\begin{align*}
     \beta^E = \sqrt{16 H \log\frac{KH \cdot N_\cV (\epsilon)\cdot N_\cF (\epsilon)}{\delta} + 56L\epsilon\cdot K / \alpha^2 + \lambda}.
\end{align*}
Then we complete the proof by taking a union bound for all $k\in[K]$ and $h\in[H]$.
\end{proof}

\begin{proof}[Proof of Lemma~\ref{lm:overline-event-plan}]
    We have $\cT_h \hat{V}_{h+1} \in \cF_h$ due to Assumption~\ref{assumption:complete}. For any function $V:~S\rightarrow[0,1]$, let $\eta_h^k(V) = r_{h}(s_h^k,a_h^k) + V(s_{h+1}^k) - \cT_h V(s_h^k,a_h^k)$. For all $f\in\cF_h$, we have
\begin{align*}
    &\sum_{i\in[K]}\frac{\hat{\ind}_h}{(\bar{\sigma}_{i,h})^2} \Big( f(s_h^i,a_h^i) - \cT_h \hat{V}_{h+1}(s_h^i,a_h^i) \Big)^2 - 2\underbrace{\sum_{i\in[K]}\frac{\hat{\ind}_h}{(\bar{\sigma}_{i,h})^2} \Big( f(s_h^i,a_h^i) - \cT_h \hat{V}_{h+1}(s_h^i,a_h^i)\big) \eta_h^k(\hat{V}_{h+1})}_{I(f,\cT_{h}\hat{V}_{h+1},\hat{V}_{h+1})} \\
    &= \sum_{i\in[K]}\frac{\hat{\ind}_h}{(\bar{\sigma}_{i,h})^2}\Big( r_{h}(s_h^i,a_h^i) + \hat{V}_{h+1}(s_{h+1}^i) - f(s_h^i,a_h^i) \Big)^2 - \sum_{i\in[K]}\frac{\hat{\ind}_h}{(\bar{\sigma}_{i,h})^2} \eta_h^k( \hat{V}_{h+1})^2.
\end{align*}
By definition, we have that 
\begin{align*}
    \sum_{i\in[K]}\frac{\hat{\ind}_h}{(\bar{\sigma}_{i,h})^2} \Big( \hat f_h(s_h^i,a_h^i) - \cT_h \hat{V}_{h+1}(s_h^i,a_h^i) \Big)^2 - 2 I(\hat{f}_h,\cT_{h}\hat{V}_{h+1},\hat{V}_{h+1}) \le 0.
\end{align*}
We decompose $I(\hat{f}_h,\cT_{h}\hat{V}_{h+1},\hat{V}_{h+1})$ into two parts:
\begin{align}\label{eq:decomposition}
    I(\hat{f}_h,\cT_{h}\hat{V}_{h+1},\hat{V}_{h+1}) & = \sum_{i\in[K]}\frac{\hat{\ind}_h}{(\bar{\sigma}_{i,h})^2} \Big( \hat f_h(s_h^i,a_h^i) - \cT_h \hat{V}_{h+1}(s_h^i,a_h^i)\Big) \eta_h^k(\hat{V}_{h+1} - V^*_{h+1}) \nonumber \\
    & \qquad + \sum_{i\in[K]}\frac{\hat{\ind}_h}{(\bar{\sigma}_{i,h})^2} \Big( \hat f_h(s_h^i,a_h^i) - \cT_h \hat{V}_{h+1}(s_h^i,a_h^i)\Big) \eta_h^k(V^*_{h+1}).
\end{align}

For the first term in \eqref{eq:decomposition}, we have
\begin{align*}
    \EE \bigg[ \frac{\hat{\ind}_h}{(\bar{\sigma}_{i,h})^2} \Big( f(s_h^i,a_h^i) - \bar{f}(s_h^i,a_h^i)\big) \eta_h^k(\hat V_{h+1} - V^*_{h+1}) \bigg] = 0.
\end{align*}
Furthermore, we can bound the maximum as following:
\begin{align*}
    &\max_{i\in[K]} \bigg|  \frac{\hat{\ind}_h}{(\bar{\sigma}_{i,h})^2} \Big( f(s_h^i,a_h^i) - \bar{f}(s_h^i,a_h^i)\big) \eta_h^k(\hat V_{h+1} - V^*_{h+1})  \bigg| \\
    & \le 2 \max_{i\in[K]} \bigg| \frac{\hat{\ind}_h}{(\bar{\sigma}_{i,h})^2}\Big( f(s_h^i,a_h^i) - \bar{f}(s_h^i,a_h^i)\big) \bigg| \\
    & \le 2 \max_{i\in[K]} \frac{\hat{\ind}_h}{(\bar{\sigma}_{i,h})^2} \sqrt{D^2_{\cF_h}(z_{i,h};z_{[i-1],h},\bar{\sigma}_{[i-1],h})\bigg(  \sum_{s \in [i-1] } \frac{1}{(\bar{\sigma}_{s,h})^2} (f(s_h^s,a_h^s) - \bar{f}(s_h^s,a_h^s))^2 +\lambda \bigg) } \\
    & \le 2 \max_{i\in[K]} \frac{1}{(\bar{\sigma}_{i,h})^2} \sqrt{D^2_{\cF_h}(z_{i,h};z_{[i-1],h},\bar{\sigma}_{[i-1],h})\bigg(  \sum_{s \in [i-1] } \frac{\hat{\ind}_h}{(\bar{\sigma}_{s,h})^2} (f(s_h^s,a_h^s) - \bar{f}(s_h^s,a_h^s))^2 + \lambda\bigg) } \\
    & \le 2 \cdot \gamma^{-2}  \sqrt{ \sum_{s \in [K] } \frac{\hat{\ind}_h}{(\bar{\sigma}_{s,h})^2} (f(s_h^s,a_h^s) - \bar{f}(s_h^s,a_h^s))^2 + \lambda },
\end{align*}
where the first inequality is due to bounded total rewards assumption, the second inequality holds due to Definition~\ref{def:ged}, and the last inequality holds due to Line~\ref{ln:hat_sigma} in Algorithm~\ref{alg:main} and Definition~\ref{def:oracle}. 

We further define $ \text{var} (V - V^*_{h+1}) $ as
\begin{align*}
    \text{var} (V - V^*_{h+1})&:= \sum_{i\in[K]} \EE \Big[\frac{\hat{\ind}_h}{(\bar{\sigma}_{i,h})^4} \Big( f(s_h^i,a_h^i) - \bar f(s_h^i,a_h^i)\Big)^2 \eta_h^k( \hat V_{h+1} - V^*_{h+1})^2 \Big]\\
    & \le  L^2 K / \alpha^4 .
\end{align*}
{By the definition of the indicator function}, we have
\begin{align*}
    \text{var} (V - V^*_{h+1}) \le \frac{4}{\eta} \sum_{i\in[K]}\frac{\hat{\ind}_h}{(\bar{\sigma}_{i,h})^2} \Big( f(s_h^i,a_h^i) - \bar f(s_h^i,a_h^i)\Big)^2 
\end{align*}
For fixed $f$, $\bar f$, by applying Lemma~\ref{lemma:freedman-variant} with $V^2 = L^2 K / \alpha^4, M = 2L/\alpha^2, v = \eta^{-1/2},m=v^2$, and probability at least {$ 1 - \delta / (N_{\cF}(\epsilon)^2 N_{\cV}(\epsilon) H )$} we have
\begin{align*}
& \sum_{i \in [K]}\frac{\hat{\ind}_h}{(\bar{\sigma}_{i,h})^2} \Big( f(s_h^i,a_h^i) - \bar f(s_h^i,a_h^i)\Big) \eta_h^k( V - V^*_{h+1}) \\
& \le \iota \sqrt{2 \left(2\text{var}(V - V_{h + 1}^*) + \eta^{-1}\right)} \\
&\ + \frac{2}{3} \iota^2 \bigg(4 \gamma^{-2} \sqrt{\sum_{i\in[K]} \frac{\hat\ind_{h}}{(\bar \sigma_{i, h})^2} \left(f(s_h^i, a_h^i) - \bar f(s_h^i, a_h^i)\right)^2 + \lambda} + \eta^{-1}\bigg),
\end{align*}
where  
\begin{align*} 
\iota^2(k, h, \delta) &:= \log \frac{N_{\cF}(\epsilon)^2 \cdot N_{\cV}(\epsilon) \cdot (\log(L^2 K \eta / \alpha^4) +2 ) \cdot (\log(2L\eta / \alpha^2) + 2)}{\delta / H} 
\end{align*}
Using a union bound over all {$(f, \bar f, V) \in \cC (\cF_h, \epsilon) \times \cC(\cF_h,\epsilon) \times \cC(\cV_{h=1}, \epsilon)$}, we have the inequality above holds for all such $f,\bar f, V$ with probability at least $1 - \delta / H$. There exist a $V_{h+1}^c$ in the $\epsilon$-net such that $\| \hat V_{h+1} - V_{h+1}^c\| \le \epsilon$. Then we have
\begin{align}
    & \sum_{i \in [K]}\frac{\hat{\ind}_h}{(\bar{\sigma}_{i,h})^2} \Big( \hat f_h (s_h^i,a_h^i) - \cT_h \hat{V}_{h+1} (s_h^i,a_h^i)\Big) \eta_h^k( \hat V_{h+1} - V^*_{h+1}) \nonumber \\
    & \le O\bigg({\iota(k, h, \delta)}\eta^{-1/2} + {\iota(k, h, \delta)^2}\gamma^{-2}\bigg) \cdot\sqrt{\sum_{\tau\in[K]} \frac{\hat \ind_{h}}{(\bar \sigma_{\tau, h})^2} (\hat f_{h}(s_h^\tau, a_h^\tau) -  \cT_h V_{h+1}(s_h^\tau, a_h^\tau))^2 + \lambda} \nonumber \\
    &\  + O(\epsilon k L / \alpha^2) + O({\iota^2(k, h, \delta)}\eta^{-1}) + O({\iota(k, h, \delta)}\eta^{-1/2}).\label{eq:plan-event-1}
\end{align}
For the second term in \eqref{eq:decomposition}, applying Lemma~\ref{lemma:hoeffding-variant}, for fixed $f$, $\bar f$, and $V^*_{h+1}$, with probability at least $1-\delta$, we have
\begin{align*}
& \sum_{i\in[K]}\frac{\hat \ind_{h}}{(\bar{\sigma}_{i,h})^2} \Big( f(s_h^i,a_h^i) - \bar f (s_h^i,a_h^i)\Big) \eta_h^k(V^*_{h+1}) \\
    & \le  \frac{1}{4} \sum_{i\in[K]}\frac{\hat \ind_{h}}{(\bar{\sigma}_{i,h})^2} \Big( f(s_h^i,a_h^i) - \bar f (s_h^i,a_h^i)\Big)^2 + \frac{8}{\alpha^2} \cdot \log \frac{1}{\delta}.
\end{align*}
Applying a union bound, for any $k$, with probability at least $1-\delta$, we have for all $f^a,f^b$ in the $\epsilon$-net $\cF_h$
\begin{align*}
    I(f^a,f^b,V_{h+1}^*) \le \frac{1}{4}\sum_{i\in[k-1]} \frac{\hat \ind_{i,h}}{(\bar{\sigma}_{i,h})^2} \Big( f^a(s_h^i,a_h^i) - f^b (s_h^i,a_h^i)\Big)^2 + \frac{8}{\alpha^2} \cdot \log\frac{ N_{\cF}(\epsilon)^2}{\delta}.
\end{align*}
Therefore, with probability at least $1-\delta$, we have
\begin{align}
    & I(\hat{f}_{h},\cT_{h}V_{h+1},V_{h+1}^*) \le I(f^a,f^b,V_{h+1}^*) + 8\epsilon \cdot K / \alpha^2 \nonumber \\
    & \le \frac{1}{4}\sum_{i\in[K]} \frac{\hat \ind_{h}}{(\bar{\sigma}_{i,h})^2} \Big( f^a(s_h^i,a_h^i) - f^b (s_h^i,a_h^i)\Big)^2 + \frac{8}{\alpha^2} \cdot \log\frac{\cdot N_\cF(\epsilon)^2}{\delta} + 8\epsilon \cdot k / \alpha^2 \nonumber \\
    & \le \frac{1}{4}\sum_{i\in[K]} \frac{\hat \ind_{h}}{(\bar{\sigma}_{i,h})^2} \Big( \hat{f}_{h}(s_h^i,a_h^i) - \cT_h V_{h+1} (s_h^i,a_h^i)\Big)^2 + \frac{8}{\alpha^2} \cdot \log\frac{N_\cF(\epsilon)^2}{\delta} \nonumber \\
    & \qquad + 8\epsilon \cdot k / \alpha^2 + 2L\epsilon\cdot k / \alpha^2 \nonumber \\
    & \le \frac{1}{4}\sum_{i\in[K]} \frac{\hat \ind_{h}}{(\bar{\sigma}_{i,h})^2} \Big( \hat{f}_{h}(s_h^i,a_h^i) - \cT_h V_{h+1} (s_h^i,a_h^i)\Big)^2 + \frac{8}{\alpha^2} \cdot \log\frac{ N_\cF(\epsilon)^2}{\delta} + 10L\epsilon\cdot k / \alpha^2.\label{eq:plan-event-2}
\end{align}
Taking $\eta=\log N_{\cV}(\epsilon)$, $\gamma = \tilde{O} \big(\sqrt{\log N_{\cV}(\epsilon)}\big)$ and $\alpha = 1 / \sqrt{H}$ and substituting \eqref{eq:plan-event-1} and \eqref{eq:plan-event-2} back into \eqref{eq:decomposition}, we have
\begin{align*}
& \lambda + \sum_{i \in [K]} \frac{\ind_{h}}{\bar \sigma_{i, h}^2}\left(\hat f_{h}(s_h^i, a_h^i) - {\cT}_h \hat{V}_{h + 1}(s_h^i, a_h^i)\right)^2 \\
    & \le O\bigg( H \log N_{\cF}(\epsilon)\bigg) +  O\bigg(  (\log N_{\cV}(\epsilon))^{-1} \log \frac{(\log(L^2 K / \alpha^4) +2 ) \cdot (\log(2L/ \alpha^2) + 2)}{\delta / H}  \bigg)    + O(\lambda).
\end{align*}
\end{proof}

\begin{lemma}\label{lm:over-optimism}
     On the event $\cE^E \cap \cE_h^P  $, for any $h\in[H]$, we have
    \begin{align}
        V_h^*(s;r) + V_{k,h}(s) \ge \hat{V}_h (s).\label{eq:over-optimism}
    \end{align}
\end{lemma}

\begin{lemma}\label{lm:indicator}
    On the event $\cE^{E} \cap \cE_{h+1}^P$, for each episode $k \in [K]$, we have
    \begin{align*}
        \log N_{\cV }(\epsilon) \cdot [\VV_h (\hat V_{h+1} - V^*_{h+1})](s_h^k,a_h^k) \le \sigma_{k,h}^2,
    \end{align*}
    where $\sigma_{k,h}^2 = 4 \log N_\cV(\epsilon) \cdot \min\{\hat f_{k,h} (s_h^k,a_h^k),1  \}$.
\end{lemma}

\begin{proof}[Proof of Lemma~\ref{lm:event-plan}]
    Recall that the indicator function in event $\overline{\cE}^P$ is 
    \begin{align*}
        \hat{\ind}_h = & \underbrace{\ind (V^*_{h+1}(s) \le \hat V_{h+1}(s),~\forall s \in \cS)}_{I_1} \cdot \underbrace{\ind (\hat V_{h+1}(s) \le V_{k,h+1}(s)+V^*(s;r),~\forall s \in \cS,\forall k \in  [K])}_{I_2}  \\
        & \cdot \underbrace{\ind([\VV_h (\hat V_{h+1} - V^*_{h+1})](s_h^k,a_h^k) \le \eta^{-1} \bar \sigma_{k,h}^2,~\forall k \in  [K])}_{I_3},
    \end{align*}
    where $\eta = \log N_{\cV} (\epsilon)$. Lemma~\ref{lm:over-optimism}, Lemma~\ref{lm:optimism-plan}, and Lemma~\ref{lm:indicator} indicate that $I_1 = I_2 = I_3=1$. 
\end{proof}

\begin{proof}[Proof of Lemma~\ref{lemma:opt-cover}]
   There exists an $\epsilon / 2$-net of $\cF$, denoted by $\cC(\cF_h, \epsilon / 2)$, such that for any $f \in \cF_h$, we can find $f' \in \cC(\cF, \epsilon / 2)$ such that $\|f - f'\|_\infty \le \epsilon / 2$. Also, there exists an $\epsilon / 2\beta$-net of $\cB$, $\cC(\cB, \epsilon / 2\beta)$. 

   Then we consider the following subset of $\cV_{h}$, 
   \begin{align*}
       \cV_{h}^c = \left\{V(\cdot)  = \max_{a \in \cA}  \min \big(1, f(\cdot, a) + \beta \cdot b(\cdot, a)\big) \bigg| f \in \cC(\cF_h, \epsilon / 2), b \in \cC(\cB, \epsilon / 2\beta) \right\}. 
   \end{align*}
   Consider an arbitrary $V \in \cV$ where $V = \max_{a \in \cA} \min(1, f_i(\cdot, a) + \beta \cdot b_i(\cdot, a))$. For each $f_i$, there exists $f_i^c \in \cC(\cF_h, \epsilon / 2)$ such that $\|f_i - f_i^c\|_\infty \le \epsilon / 2$. There also exists $b^c \in \cC(\cB, \epsilon / 2 \beta)$ such that $\|b_i - b^c\|_\infty\le \epsilon / 2\beta$. Let $V^c = \max_{a \in \cA} \min(1, f_i^c(\cdot, a) + \beta \cdot b^c(\cdot, a)) \in \cV^c$. It is then straightforward to check that $\|V - V^c\|_\infty \le \epsilon / 2 + \beta \cdot \epsilon / 2\beta = \epsilon$. 

   By direct calculation, we have $|\cV_h^c| = N_{\cF_h}( \epsilon / 2) \cdot N_\cB(\epsilon / 2\beta)$. 
\end{proof}

\begin{proof}[Proof of Lemma~\ref{lm:concentration-plan}]
    According to the definition of $D_\cF^2$ function, we have
\begin{align}
    &\big(\hat{f}_{k,h}(s,a)-\cT_h V_{k,h+1}(s,a)\big)^2\notag\\
    &\leq D_{\cF_h}^2(z; z_{[k - 1],h}, \bar\sigma_{[k - 1],h})\times \bigg(\lambda + \sum_{i=1}^{k-1} \frac{1}{(\bar\sigma_{i, h})^2}\left(\hat f_{k, h}(s_{h}^i, a_{h}^i) - \cT_{h} V_{k, h + 1}(s_{h}^i, a_{h}^i)\right)^2\bigg)\notag\\
    &\leq {(\beta^E)}^2 \times D_{\cF_h}^2(z; z_{[k - 1],h}, \bar\sigma_{[k - 1],h}),\notag
\end{align}
   where the first inequality holds due the definition of $D_\cF^2$ function with the Assumption \ref{assumption:complete} and the second inequality holds due to the events $\cE_{h}^{E}$. Thus, we have
   \begin{align*}
       \big|\hat{f}_{k,h}(s,a)-\cT_h V_{k,h+1}(s,a)\big|\leq {\beta^E} D_{\cF_h}(z; z_{[k - 1],h}, \bar\sigma_{[k - 1],h}).
   \end{align*}
\end{proof}

\begin{proof}[Proof of Lemma~\ref{lm:optimism-plan}]
    We prove this statement by induction. Note that $V_{H+1}^*(s;r) = \hat{V}_{H+1}(s)$. 
    Assume that the statement holds for $h+1$. If $\hat{V}_{h}(s)=1$, then the statement holds trivially for $h$; otherwise, we have for any $(s,a)\in \cS \times \cA$ that 
    \begin{align*}
        & \hat Q_h (s,a) - Q^*_h (s,a;r) \\
        & =  \hat f_h (s,a) + b_h (s,a) - [ r_h (s,a;r) + \PP_h V^*_{h+1}(s,a;r)] \\
        & =  [\hat f_h (s,a) -  r_h (s,a;r) - \PP_h \hat V_{h+1}(s,a;r)] + b_h (s,a) + \PP_h \hat V_{h+1}(s,a;r) - \PP_h V^*_{h+1}(s,a;r) \\
        & \ge [\hat f_h (s,a) -  r_h (s,a;r) - \PP_h \hat V_{h+1}(s,a;r)] + b_h (s,a) \\
        & \ge - \beta^P D_{\cF_{h}}(z;z_{[K],h},\bar{\sigma}_{[K],h}) + \beta^P \overline{D}_{\cF_{h}}(z;z_{[K],h},\bar{\sigma}_{[K],h}) \\
        & \ge 0,
    \end{align*}
    where the first inequality holds due to the induction assumption, and the second inequality holds due to Lemma~\ref{lm:concentration-plan}.
\end{proof}

In order to prove Lemma~\ref{lm:value-sum}, we need the following three lemmas.
\begin{lemma}[Simulation Lemma]\label{lm:simulation}
On the event $\underline{\cE}^E$, we have
\begin{align*}
    0 \le {V}_{k,h} (s_h^k) \le \EE_{\tau_h^k \sim d_h^{\pi^k}(s_h^k) } \min \Big\{ 3 {\beta}^E \sum_{h' = h}^{H} \overline{\cD}(z_{k,h'};z_{[k-1],h'},\overline{\sigma}_{[k-1],h'}), 1 \Big\}.
\end{align*}
\end{lemma}

\begin{lemma}\label{lemma:bonus-sum}
   [Lemma~C.13 in \citet{zhao2023nearly}]For any parameters $\beta \ge 1$ and stage $h\in [H]$, the summation of confidence radius over episode $k\in[K]$ is upper bounded by
   \begin{align*}
    & \sum_{k=1}^K\min\Big(\beta D_{\cF_h}(z; z_{[k - 1],h}, \bar\sigma_{[k - 1],h}),1\Big)\\
    & \leq (1+C\beta\gamma^2) \dim_{\alpha, K}(\cF_h) + 2 \beta \sqrt{\dim_{\alpha, K}(\cF_h)} \sqrt{\sum_{k=1}^K (\sigma_{k,h}^2+\alpha^2)},
   \end{align*}
   where $z=(s,a)$ and $z_{[k - 1],h}=\{z_{1,h},z_{2,h},..,z_{k-1,h}\}$.
\end{lemma}

\begin{lemma} \label{lm:sigma_sum}
    Under event $\underline{\cE}^E$, we have
    \begin{align*}
        \sum_{k=1}^K \sum_{h=1}^H {\sigma}_{k,h}^2 
        &\le 2304 C^2 H^3 (\log N_{\cV}(\epsilon))^2  (\beta^E)^2 \dim_{\alpha, K}(\cF) \\
        & \qquad + 48 H^2 \log N_{\cV}(\epsilon)  (1+C\beta^E\gamma^2) \dim_{\alpha, K}(\cF_h) + 16 H \log N_{\cV}(\epsilon) \sqrt{2 H K \log(H/\delta)} + K.
    \end{align*}
\end{lemma}
Now we can prove Lemma~\ref{lm:value-sum}.
\begin{proof}[Proof of Lemma~\ref{lm:value-sum}]
    We have
    \begin{align*}
        \sum_{k=1}^K V_{k,1} (s_1^k) 
        & \le \sum_{k=1}^K \EE_{\tau_h^k \sim d_h^{\pi^k}(s_h^k) } \min \Big\{ 3 {\beta}^E \sum_{h' = 1}^{H} \overline{\cD}(z_{k,h};z_{[k],h},\overline{\sigma}_{[k],h}), 1 \Big\}\\
        & \le \sum_{k=1}^K \sum_{h = 1}^{H} \EE_{\tau_h^k \sim d_h^{\pi^k}(s_h^k) } \min \Big\{ 3 {\beta}^E  \overline{\cD}(z_{k,h};z_{[k-1],h},\overline{\sigma}_{[k-1],h}), 1 \Big\} \\
        & \le H (1+4C\beta^E\gamma^2) \dim_{\alpha, K}(\cF_h) + 8 \beta^E \sqrt{\dim_{\alpha, K}(\cF)} \sqrt{ H \sum_{k=1}^K \sum_{h=1}^H (\sigma_{k,h}^2+\alpha^2)} \\
        & = O(\beta^E \sqrt{ K H \dim_{\alpha, K}(\cF)}),
    \end{align*}
    where the first inequality follows from Lemma~\ref{lm:simulation}, the third inequality follows from Lemma~\ref{lemma:bonus-sum}, and the last equality holds due to Lemma~\ref{lm:sigma_sum}.
\end{proof}

\begin{proof}[Proof of Lemma~\ref{lm:deviation-sum}]
    Denote $\Delta_k = \EE_{s \sim \mu}\big[ \tilde{V}_1^*(s;r_k)  \big] - \tilde{V}_1^*(s_1^k;r_k)$. By Azuma-Hoeffding inequality (Lemma~\ref{lemma:hoeffding}), we have
    \begin{align*}
        \Big| \sum_{k=1}^K \Delta_k \Big| \le \sqrt{2K \log(1/\delta)}.
    \end{align*}
\end{proof}

\begin{lemma}\label{lm:optimism-exp}
    On the event $\underline \cE^{E}$, for any $k\in[K]$ and $h\in[H]$,  we have
    \begin{align*}
        \tilde{V}_h^*(s;r_k) \le V_{k,h}(s),\quad \forall s \in \cS.
    \end{align*}
\end{lemma}

\begin{proof}[Proof of Lemma~\ref{lm:value_bonus}]
    Since $ \beta^E = O(\sqrt{H\log N_{\cV}(\epsilon)})$ and $ \beta^P = O(\sqrt{ H \log N_{\cF} (\epsilon)})$, for some constant $c$, we have
    \begin{align*}
        \beta^E \ge c \sqrt{ \log N_{\cV}(\epsilon)/ \log N_{\cF} (\epsilon) } \cdot \beta^P.
    \end{align*}
    Therefore, for any $h\in [H]$, we have $r_{k,h}(\cdot, \cdot) \ge r_{K,h}(\cdot, \cdot) \ge c \sqrt{ \log N_{\cV}({\epsilon}) / \log N_{\cF} (\epsilon) } \cdot  b_h (\cdot,\cdot)$. Hence,
    \begin{align*}
        & c \sqrt{ \log N_{\cV}({\epsilon}) / \log N_{\cF} (\epsilon) } \cdot \EE_{s\sim\mu} \Big[ \tilde{V}_1^* (s;b)\Big] \\
        & =  \EE_{s\sim\mu} \Big[  \tilde{V}_1^* (s; c \sqrt{ \log N_{\cV}({\epsilon}) /  \log N_{\cF} (\epsilon) } \cdot b)\Big] \\
        & \le \EE_{s\sim\mu} \Big[ \tilde{V}_1^* (s;r_k)\Big] / K \\
        & = \Big[ \sum_{k=1}^K \tilde{V}_1^* (s_1^k;r_k) + \sum_{k=1}^K \Big[ \EE_{s\sim\mu} \Big[ \tilde{V}_1^* (s;r_k)\Big] -  \tilde{V}_1^* (s_1^k;r_k)\Big]   \Big] / K \\
        & \le \Big(  \sum_{k=1}^K \tilde{V}_1^* (s;r_k) \Big)/ K + \sqrt{2\log(1/\delta)/K} \\
        & \le \Big(  \sum_{k=1}^K {V}_{k,1} (s;r_k) \Big)/ K + \sqrt{2\log(1/\delta)/K} \\
        & = O\Big(\beta^E  \sqrt{H \dim_{\alpha,K} (\cF) / K}\Big),
    \end{align*}
    where the second inequality follows from Lemma~\ref{lm:deviation-sum}, and the third inequality follows from Lemma~\ref{lm:optimism-exp}. Therefore, we have
    \begin{align*}
        \EE_{s\sim\mu} \Big[ \tilde{V}_1^* (s;b)\Big] = O\Big(\beta^E   \sqrt{ H
 \dim_{\alpha,K} (\cF) / K} \sqrt{ \log N_{\cF} (\epsilon) / \log N_{\cV}(\epsilon) } \Big).
    \end{align*}
\end{proof}

\section{Proofs of Lemmas in Appendix~\ref{app:level1}}\label{app:level2}

\begin{proof}[Proof of Lemma~\ref{lm:over-optimism}]
We see that
\begin{align*}
    Q^*(\cdot,\cdot ; r) & = r_h(\cdot,\cdot) + \PP_h V_{h+1}(\cdot,\cdot;r), \\
    Q_{k,h}(\cdot,\cdot) & = \min\{ \hat{f}_{k,h}(\cdot,\cdot) + b_{k,h}(\cdot,\cdot), 1\}, \\
    \hat{Q}_{h}(\cdot,\cdot) & = \min \{ \hat{f}_h (\cdot,\cdot) + b_h (\cdot,\cdot), 1\}.
\end{align*}
We prove this statement by induction. Note that $V_{H+1}^*(s;r) + V_{k,H+1}(s) = \hat{V}_{H+1} (s) = 0$. Assume the statement holds for $h+1$. By definition, we have
\begin{align*}
    Q_{h}^* (s,a;r) + 1 \ge \hat{Q}_h(s,a).
\end{align*}
Therefore, we only need to prove
\begin{align*}
    Q_{h}^* (s,a;r) + \hat{f}_{k,h}(s,a)+ b_{k,h}(s,a)  - \hat{Q}_h(s,a) \ge 0.
\end{align*}
We have
\begin{align*}
    & Q_{h}^* (s,a;r) + \hat{f}_{k,h}(s,a)+ b_{k,h}(s,a)  - \hat{Q}_h(s,a) \\
    & = r_h (s,a) + \PP_h V_{h+1}^* (s,a;r) +  \hat{f}_{k,h}(s,a) + b_{k,h}(s,a) - \min \{ \hat{f}_h (s,a) + b_h (s,a), 1\} \\
    & \ge r_h (s,a) + \PP_h V_{h+1}^* (s,a;r) +  \hat{f}_{k,h}(s,a) + b_{k,h}(s,a) -  (\hat{f}_h (s,a) + b_h (s,a))\\
    & = \PP_h V_{h+1}^* (s,a;r) + \PP_h V_{k,h+1} (s,a) - \PP_{h} \hat{V}_{h+1}(s,a) + \hat{r}_{k,h}(s,a) + b_{k,h}(s,a) - b_h (s,a)\\
    & \qquad  + (\hat{f}_{k,h}(s,a) - \hat{r}_{k,h}(s,a) - \PP_h V_{k,h+1} (s,a)) + (r_h (s,a) + \PP_h \hat{V}_h(s,a) - \hat{f}_h (s,a))\\
    & \ge \hat{r}_{k,h}(s,a) + b_{k,h}(s,a) - b_h (s,a) + (\hat{f}_{k,h}(s,a) - \hat{r}_{k,h}(s,a) - \PP_h V_{k,h+1} (s,a)) \\
    & \qquad  + (r_h (s,a) + \PP_h \hat{V}_h(s,a) - \hat{f}_h (s,a)) \\
    & \ge 3 {\beta}^E \overline{\cD}_{\cF_h}(z;z_{[k-1],h},\overline{\sigma}_{[k-1],h}) - {\beta}^P \overline{\cD}_{\cF_h}(z;z_{[K],h},\overline{\sigma}_{[K],h}) - {\beta}^E {\cD}_{\cF_h}(z;z_{[k-1],h},\overline{\sigma}_{[k-1],h}) \\
    & \qquad - {\beta}^P {\cD}_{\cF_h}(z;z_{[K],h},\overline{\sigma}_{[K],h})\\
    & \ge 0,
\end{align*}
where the second inequality holds due to induction assumption, the third inequality holds by high probability events, and the last inequality holds by ${\beta}^E \ge {\beta}^P$, $\bar{\cD}_{\cF_h}(z;z_{[k],h},\overline{\sigma}_{[k],h})$ decreasing with $k$, and Definition~\ref{def:bonus-oracle}.
\end{proof}

\begin{lemma}\label{lm:concentration-exp}
On the event $\underline \cE^{E}$, we have
\begin{align*}
        |\hat{f}_{k,h}(s,a) - \cT_h V_{k,h+1} | \le {\beta}^E D_{\cF_{h}}(z;z_{[k-1],h},\bar{\sigma}_{[k-1],h})
\end{align*}   
\end{lemma}

\begin{proof}[Proof of Lemma~\ref{lm:indicator}]
    We have Lemma~\ref{lm:optimism-plan} and \ref{lm:over-optimism} both hold on $\cE_{h+1}^P$. Therefore, we have
    \begin{align*}
        & [\VV_h (\hat V_{h+1} - V^*_{h+1})](s_h^k,a_h^k) \\
        & \le [\PP_h (\hat V_{h+1} - V^*_{h+1})^2] (s_h^k,a_h^k) \\
        & \le 2 [\PP_h (\hat V_{h+1} - V^*_{h+1})] (s_h^k,a_h^k) \\
        & \le 2 [\PP_h V_{k,h+1}] (s_h^k,a_h^k) \\
        & = 2 (\cT_h V_{k,h+1} (s_h^k,a_h^k) - r_{k,h}(s_h^k,a_h^k) ) \\
        & \le 2 ( \hat f_{k,h} (s_h^k,a_h^k) + \beta^E D_{\cF_h} (z_{k,h};z_{[k-1],h},\bar{\sigma}_{[k-1],h}) - \beta^E \overline{D}_{\cF_h} (z_{k,h};z_{[k-1],h},\bar{\sigma}_{[k-1],h})) \\
        & \le 2\hat f_{k,h} (s_h^k,a_h^k),
    \end{align*}
    where the second inequality holds due to Lemma~\ref{lm:optimism-plan} and $\hat V_{h+1}, V^*_{h+1} \in [0,1]$, the third inequality holds due to Lemma~\ref{lm:over-optimism}, the fourth inequality holds due to Lemma~\ref{lm:concentration-exp}, and the last inequality holds due to Definition~\ref{def:bonus-oracle}.
\end{proof}

\begin{proof}[Proof of Lemma~\ref{lm:simulation}]
According to Algorithm~\ref{alg:main}, we have that
    \begin{align*}
         Q_{k,h}(\cdot,\cdot)  & = \min\{ \hat{f}_{k,h}(\cdot,\cdot) + b_{k,h}(\cdot,\cdot),1\},\\
        V_{k,h}(\cdot) & = \max_a Q_{k,h}(\cdot,a),\\
        a_h^k &=\pi_h^k(s_h^k) = \argmax_{a} Q_{k,h}(s_h^k,a).
    \end{align*}
For all $k$ and all $h$, we have that
\begin{align*}
     V_{k,h}(s_h^k) & =Q_{k,h}(s_h^k,a_h^k) \\
    &\le \hat{f}_{k,h}(s_h^k,a_h^k) + b_{k,h}(s_h^k,a_h^k) \\
    &= 2 {\beta}^E \overline{\cD}(z_{k,h};z_{[k-1],h},\overline{\sigma}_{[k-1],h}) + (\hat{f}_{k,h}(s_h^k,a_h^k) - \cT_{h} V_{k,h+1}(s_h^k,a_h^k)) + \cT_{h} V_{k,h+1}(s_h^k,a_h^k)\\
    & \cdots \\
    & = \EE_{\tau_h^k \sim d_h^{\pi^k}(s_h^k) } \sum_{h'=h}^{H} \Big[ (\hat{f}_{k,h'}(s_{h'}^k,a_{h'}^k) - \cT_{h} V_{k,h'+1}(s_{h'}^k,a_{h'}^k)) + 2 {\beta}^E \overline{\cD}(z_{k,h'};z_{[k-1],h'},\overline{\sigma}_{[k-1],h'})  \Big] \\
    & \le \EE_{\tau_h^k \sim d_h^{\pi^k}(s_h^k) } \sum_{h'=h}^{H} 3 {\beta}^E \overline{\cD}(z_{k,h'};z_{[k-1],h'},\overline{\sigma}_{[k-1],h'}),
\end{align*}
where the last inequality holds due to Lemma~\ref{lm:concentration-exp} and Definition~\ref{def:bonus-oracle}.
\end{proof}

\begin{lemma}\label{lm:expection-sum}
    On the event $\underline{\cE}^E$, with probability at least $1-\delta$, 
    \begin{align*}
        \sum_{k=1}^K \sum_{h=1}^H \PP_h V_{k,h+1}(s_h^k,,a_h^k) & \le H \sum_{k=1}^K \sum_{h=1}^H \min \Big\{   4 \beta^E D_{\cF_h}(z_{k,h};z_{[k-1],h},\bar{\sigma}_{[k-1],h}),1 \Big\}  \\
        & \qquad + (H + 1) \sqrt{2 H K \log(1/\delta )}
    \end{align*}
\end{lemma}

\begin{proof}[Proof of Lemma~\ref{lm:sigma_sum}]
    Recall ${\sigma}_{k,h}^2 = 4 \log N_\cV(\epsilon) \cdot \min\{\hat f_{k,h} (s_h^k,a_h^k),1  \}$. We have
    \begin{align*}
        \sum_{k=1}^K \sum_{h=1}^H {\sigma}_{k,h}^2 
        & = 4 \log N_{\cV}(\epsilon) \sum_{k=1}^K \sum_{h=1}^H  \min\{\hat f_{k,h} (s_h^k,a_h^k),1  \} \\
        & \le 4 \log N_{\cV}(\epsilon) \sum_{k=1}^K \sum_{h=1}^H \min\{ \cT_h V_{k,h+1}(s_h^k,a_h^k)  +  \beta^E D_{\cF_h}(z_{k,h};z_{[k-1].h},\bar{\sigma}_{[k-1],h}),1  \} \\
        & \le 4 \log N_{\cV}(\epsilon) \sum_{k=1}^K \sum_{h=1}^H \min\{ \PP_h V_{k,h+1}(s_h^k,a_h^k)  + 2 \beta^E \overline{\cD}_{\cF_h}(z_{k,h};z_{[k-1].h},\bar{\sigma}_{[k-1],h}),1  \} \\
        & \le 4 \log N_{\cV}(\epsilon) \sum_{k=1}^K \sum_{h=1}^H  \PP_h V_{k,h+1}(s_h^k,a_h^k)   + 8 \log N_{\cV}(\epsilon) \sum_{k=1}^K \sum_{h=1}^H\{  \beta^E \overline{\cD}_{\cF_h}(z_{k,h};z_{[k-1],h},\bar{\sigma}_{[k-1],h}),1  \} \\
        & \le  24 H \log N_{\cV}(\epsilon) 
        \underbrace{\sum_{k=1}^K \sum_{h=1}^H \min \Big\{  \beta^E \overline{\cD}_{\cF_h}(z_{k,h};z_{[k-1],h},\bar{\sigma}_{[k-1],h}),1 \Big\}}_{I}    + 8 H \log N_{\cV}(\epsilon) \sqrt{2 H K \log(H/\delta)},
    \end{align*}
    where the first inequality holds due to Lemma~\ref{lm:concentration-exp}, the second inequality holds due to Definition~\ref{def:bonus-oracle}, and the last inequality holds due to Lemma~\ref{lm:expection-sum}. For the term $I$, we further have
    \begin{align*}
        & \sum_{k=1}^K \sum_{h=1}^H \min \Big\{\beta^E \overline{\cD}_{\cF_h}(z_{k,h};z_{[k-1],h},\bar{\sigma}_{[k-1],h}),1 \Big\} \\
        & \le \sum_{k=1}^K \sum_{h=1}^H \min \Big\{ C \beta^E \cD_{\cF_h}(z_{k,h};z_{[k-1],h},\bar{\sigma}_{[k-1],h}),1 \Big\} \\
        & \le  \sum_{h=1}^H (1+ C\beta^E\gamma^2) \dim_{\alpha, K}(\cF_h) + 2 C \beta^E \sum_{h=1}^H \sqrt{\dim_{\alpha, K}(\cF_h)} \sqrt{\sum_{k=1}^K (\sigma_{k,h}^2+\alpha^2)} \\
        & \le H (1+C\beta^E\gamma^2) \dim_{\alpha, K}(\cF) + 2 C \beta^E \sqrt{\sum_{h=1}^H\dim_{\alpha, K}(\cF_h)} \sqrt{\sum_{k=1}^K \sum_{h=1}^H (\sigma_{k,h}^2+\alpha^2)} \\
        & \le H (1 + C\beta^E\gamma^2) \dim_{\alpha, K}(\cF_h) + 2 C \beta^E \sqrt{\dim_{\alpha, K}(\cF)} \sqrt{ H \sum_{k=1}^K \sum_{h=1}^H (\sigma_{k,h}^2+\alpha^2)},
    \end{align*}
where the first inequality holds due to Definition~\ref{def:bonus-oracle}, the second inequality holds due to Lemma~\ref{lemma:bonus-sum}, the third inequality holds due to Cauchy-Schwarz inequality. Therefore, we can get
    \begin{align*}
        \sum_{k=1}^K \sum_{h=1}^H {\sigma}_{k,h}^2 
        & \le 24 H^2 \log N_{\cV}(\epsilon)  (1+C\beta^E\gamma^2) \dim_{\alpha, K}(\cF_h) \\
        & \qquad + 48 CH \log N_{\cV}(\epsilon)  \beta^E \sqrt{\dim_{\alpha, K}(\cF)} \sqrt{ H \sum_{k=1}^K \sum_{h=1}^H (\sigma_{k,h}^2+\alpha^2)} + 8 H \log N_{\cV}(\epsilon) \sqrt{2 H K \log(H/\delta)}.
    \end{align*}
Since $x \le a \sqrt{x} +b $ implies $x \le a^2 + 2b$, taking $\alpha = 1 / \sqrt{H}$, we have that
    \begin{align*}
        \sum_{k=1}^K \sum_{h=1}^H {\sigma}_{k,h}^2 
        &\le 2304 C^2 H^3 (\log N_{\cV}(\epsilon))^2  (\beta^E)^2 \dim_{\alpha, K}(\cF) \\
        & \qquad + 48 H^2 \log N_{\cV}(\epsilon)  (1+C\beta^E\gamma^2) \dim_{\alpha, K}(\cF_h) + 16 H \log N_{\cV}(\epsilon) \sqrt{2 H K \log(H/\delta)} + K.
    \end{align*}
\end{proof}

\begin{proof}[Proof of Lemma~\ref{lm:optimism-exp}]
    We prove this statement by induction. Note that $\tilde V_{H+1}^*(s;r_k) = {V}_{k,H+1}(s) = 0$. 
    Assume that the statement holds for $h+1$. If ${V}_{k,h}(s)=1$, then the statement holds trivially for $h$; otherwise, we have for any $(s,a)\in \cS \times \cA$ that 
    \begin{align*}
        & \hat Q_{k,h} (s,a) - \tilde Q^*_h (s,a;r_k) \\
        & \ge  \hat f_{k,h} (s,a) + b_{k,h} (s,a) - [ r_{k,h} (s,a;r) + \PP_h V^*_{h+1}(s,a;r)] \\
        & =  [\hat f_{k,h} (s,a) -  r_{k,h} (s,a;r) - \PP_h  V_{k,h+1}(s,a;r)] + b_{k,h} (s,a) + \PP_h V_{k,h+1}(s,a;r) - \PP_h V^*_{h+1}(s,a;r) \\
        & \ge [\hat f_{k,h} (s,a) -  r_{k,h} (s,a;r) - \PP_h  V_{k,h+1}(s,a;r)] + b_{k,h} (s,a) \\
        & \ge - \beta^E D_{\cF_{h}}(z;z_{[K],h},\bar{\sigma}_{[K],h}) + 2 \beta^E \overline{D}_{\cF_{h}}(z;z_{[K],h},\bar{\sigma}_{[K],h}) \\
        & \ge 0,
    \end{align*}
    where the first inequality holds due to Definition~\ref{def:truncated_optimal}, the second inequality holds due to induction hypothesis, the third inequality holds due to Lemma~\ref{lm:concentration-exp}, and the forth inequality holds due to Definition~\ref{def:bonus-oracle}.
\end{proof}

\section{Proof of Lemmas in Appendix~\ref{app:level2}}

\begin{proof}[Proof of Lemma~\ref{lm:concentration-exp}]
    According to the definition of $D_\cF^2$ function, we have
\begin{align}
    &\big(\hat{f}_{k,h}(s,a)-\cT_h V_{k,h+1}(s,a)\big)^2\notag\\
    &\leq D_{\cF_h}^2(z; z_{[k - 1],h}, \bar\sigma_{[k - 1],h})\times \bigg(\lambda + \sum_{i=1}^{k-1} \frac{1}{(\bar\sigma_{i, h})^2}\left(\hat f_{k, h}(s_{h}^i, a_{h}^i) - \cT_{h} V_{k, h + 1}(s_{h}^i, a_{h}^i)\right)^2\bigg)\notag\\
    &\leq {(\beta^E)}^2 \times D_{\cF_h}^2(z; z_{[k - 1],h}, \bar\sigma_{[k - 1],h}),\notag
\end{align}
   where the first inequality holds due the definition of $D_\cF^2$ function with the Assumption \ref{assumption:complete} and the second inequality holds due to the events $\cE_{h}^{E}$. Thus, we have
   \begin{align*}
       \big|\hat{f}_{k,h}(s,a)-\cT_h V_{k,h+1}(s,a)\big|\leq {\beta^E} D_{\cF_h}(z; z_{[k - 1],h}, \bar\sigma_{[k - 1],h}).
   \end{align*}
\end{proof}

\begin{proof}[Proof of Lemma~\ref{lm:expection-sum}]
    By Lemma~\ref{lemma:hoeffding}, we have
    \begin{align*}
        \sum_{k=1}^K \sum_{h=1}^H \PP_h V_{k,h+1}(s_{h}^k,a_{h}^k) & = \sum_{k=1}^K \sum_{h=1}^H V_{k,h+1}(s_{h+1}^k) + \sum_{k=1}^K \sum_{h=1}^H(\PP_h V_{k,h+1}(s_h^k,a_h^k) - V_{k,h+1}(s_{h+1}^k)) \\
        & \le \sum_{k=1}^K \sum_{h=1}^H V_{k,h+1}(s_{h+1}^k) + \sqrt{2 K H \log(1/\delta)}.
    \end{align*}
    Then, under event $\underline{\cE}^E$, we have
    \begin{align*}
        V_{k,h}(s_{h}^k) & =  Q_{k,h}(s_{h}^k,a_{h}^k) \\
        & = \min \{\hat f_{k,h}(s_{h}^k,a_{h}^k) + 2 \beta^E \overline{\cD}_{\cF_h}(z_{k,h};z_{[k-1],h},\bar{\sigma}_{[k-1],h}), 1 \} \\
        & \le \min \{\PP_h V_{k,h+1}(s_{h}^k,a_{h}^k) + 4 \beta^E \overline{\cD}_{\cF_h}(z_{k,h};z_{[k-1],h},\bar{\sigma}_{[k-1],h}), 1 \} \\
        & = \min\{ V_{k,h+1}(s_h^k) + ( \PP_h V_{k,h+1}(s_h^k, a_h^k) - V_{k,h+1}(s_h^k)) + 4 \beta^E \overline{\cD}_{\cF_h}(z_{k,h};z_{[k-1],h},\bar{\sigma}_{[k-1],h}),1\},
    \end{align*}
    where the inequality holds due to Lemma~\ref{lm:concentration-exp} and Definition~\ref{def:bonus-oracle}. Therefore, for fixed $h$, we have
    \begin{align*}
        \sum_{k=1}^K V_{k,h}(s_{h}^k) 
        & \le \sum_{k=1}^K \min \Big\{ \sum_{h'=h}^H \big[ 4 \beta^E \overline{\cD}_{\cF_{h'}}(z_{k,h'};z_{[k-1],h'},\bar{\sigma}_{[k-1],h'}) \\
        & \qquad + ( \PP_h V_{k,h'+1}(s_{h'}^k, a_{h'}^k) - V_{k,h'+1}(s_{h'}^k)) \big],1  \Big\} \\
        & \le \sum_{k=1}^K \sum_{h'=h}^H \min \Big\{   4 \beta^E \overline{\cD}_{\cF_{h'}}(z_{k,h'};z_{[k-1],h'},\bar{\sigma}_{[k-1],h'}),1 \Big\} \\
        & \qquad + \sum_{k=1}^K \sum_{h'=h}^H ( \PP_h V_{k,h'+1}(s_{h'}^k, a_{h'}^k) - V_{k,h'+1}(s_{h'}^k)) \\
        & \le \sum_{k=1}^K \sum_{h'=h}^H \min \Big\{   4 \beta^E \overline{\cD}_{\cF_{h'}}(z_{k,h'};z_{[k-1],h'},\bar{\sigma}_{[k-1],h'}),1 \Big\} + \sqrt{2 H K \log(1/\delta )}, \\
    \end{align*}
    where the first inequality holds due to induction, and the last inequality holds due to Lemma~\ref{lemma:hoeffding}.
    Hence, by combining the above two inequalities, we have
    \begin{align*}
        & \sum_{k=1}^K \sum_{h=1}^H \PP_h V_{k,h+1}(s_{h}^k,a_{h}^k) \\
        & \le \sum_{k=1}^K \sum_{h=1}^H V_{k,h+1}(s_{h+1}^k) + \sqrt{2 K H \log(1/\delta)} \\
        & \le H \sum_{k=1}^K \sum_{h=1}^H \min \Big\{   4 \beta^E \overline{\cD}_{\cF_h}(z_{k,h};z_{[k-1],h},\bar{\sigma}_{[k-1],h}),1 \Big\} + (H + 1) \sqrt{2 H K \log(1/\delta )}.
    \end{align*}
\end{proof}

\section{Auxilliary Lemmas}
\begin{lemma}[Self-normalized bound for scalar-valued martingales] \label{lemma:hoeffding-variant}
   Consider random variables $\left(v_{n} | n\in\mathbb{N}\right)$ adapted
   to the filtration $\left(\mathcal{H}_{n}:\, n=0,1,...\right)$. Let $\{\eta_i\}_{i = 1}^ \infty$ be a sequence of real-valued random variables which is $\cH_{i + 1}$-measurable and is conditionally $\sigma$-sub-Gaussian. Then for an arbitrarily chosen $\lambda > 0$, for any $\delta > 0$, with probability at least $1 - \delta$, it holds that \begin{align*} 
       \sum_{i = 1}^n \eta_i v_i \le \frac{\lambda \sigma^2}{2} \cdot \sum_{i = 1}^n v_i^2 + \log(1 / \delta) / \lambda \ \ \ \ \ \ \ \forall n \in \NN. 
   \end{align*}
\end{lemma}

\begin{lemma}[Corollary 2, \citet{agarwal2022vo}]\label{lemma:freedman-variant}
   Let $M>0,V>v>0$ be constants, and $\{x_i\}_{i\in[t]}$ be stochastic process adapted to a filtration $\{\cH_i\}_{i\in[t]}$. Suppose $\EE[x_i|\cH_{i-1}]=0$,  $|x_i|\le M$ and $\sum_{i\in[t]}\EE[x_i^2|\cH_{i-1}]\le V^2$ almost surely. Then for any $\delta,\epsilon>0$, let $\iota = \sqrt{\log\frac{\left(2\log(V/v)+2\right)\cdot\left(\log(M/m)+2\right)}{\delta}}$ we have
   \[
\PP\Bigg(\sum_{i\in[t]}x_i>\iota\sqrt{2\bigg(2\sum_{i\in[t]}\EE[x_i^2|\cH_{i-1}]+ v^2\bigg)}+\frac{2}{3}\iota^2\bigg(2\max_{i\in[t]}|x_i|+m\bigg) \Bigg)\le \delta.
   \]
\end{lemma}

\begin{lemma}[Azuma-Hoeffding Inequality]\label{lemma:hoeffding}
    Let $\{x_i\}_{i=1}^n$ be a martingale difference sequence with respect to a filtration $\{ \cG_i\}_{i=1}^{n+1}$ such that $|x_i| \le M$ almost surely. That is, $x_i$ is $\cG_{i+1}$-measurable and $\EE[x_i | \cG_i]$ a.s. Then for any $0 < \delta < 1$, with probability at least $1-\delta$,
    \begin{align*}
        \sum_{i=1}^n x_i \le M \sqrt{2n \log (1 / \delta)}.
    \end{align*}
\end{lemma}
\section{Experiment details}\label{app:experiment_details}
\subsection{Details of exploration algorithm}
We present the practical algorithm in this subsection. We start by introducing the notation $\bphi_i$ as the parameter for the $i$-th $Q$ networks, which is a three-layer MLP with 1024 hidden size, same as other benchmark algorithms implemented in URLB~\citep{laskin2021urlb}. For the ease of presentation, we ignore the $Q$ network as $Q_{\bphi_i}$ as $Q_i$ and the target network $Q_{\bar \bphi_i}$ as $\bar Q_i$ when there is no confusion. We initialize the parameters in $\bphi_i$ using Kaiming distribution~\citep{he2015delving}.

The algorithm works in the discounted MDP with the discounted factor $\gamma$. For each $t$ in training steps, the algorithm updates the $t \% N$-th $Q$ function by taking the gradient descent regarding the loss function 
\begin{align}
    \cL(\bphi_{t \% N}) = \sum_{(s, a, s') \in \cB} \frac{1}{\sigma^2(s, a)}\bigg(Q_{t \% N}(s, a) - \Big(r_{\text{int}}(s, a) + \gamma Q_{\text{target}}(s, a) + b(s, a)\Big)\bigg)^2, \label{eqq:1}
\end{align}
where the target $Q$ function is the average of $N$ target $Q$ network, i.e., $Q_{\text{target}}(s, a) = \sum_{i \in [N]} \bar Q_i(s, a) / N$, $\cB$ is the minibatch randomly sampled from replay buffer $\cD$. We encourage the diversity of different $Q$ function by using different batch $\cB$ for updating different $Q$ functions. As the key components of our algorithm, weighted regression $\sigma^2(s, a)$; intrinsic reward $r_{\text{int}}(s, a)$, exploration bonus $b(s, a)$ is calculated based on the variance of the target $Q$ network across $\bar Q_i$ instances:
\begin{align}
    \sigma^2(s, a) = \mathrm{Var}[\bar Q_i(s, a)];\qquad r_{\text{int}}(s, a) = (1 - \gamma)\sqrt{\mathrm{Var}[\bar Q_i(s, a)]};\qquad b(s, a) = \beta \sqrt{\mathrm{Var}[\bar Q_i(s, a)]}, \label{eqq:2}
\end{align}
where we simply set $\beta = 1$ to align with our theory, the factor $(1 - \gamma)$ before the intrinsic reward is because we want to balance the horizon $1 / H \approx (1 - \gamma)$ in the setting. The reason for choosing the target $Q$ function $\bar Q_i$ instead of the updating $Q$ function is to update the intrinsic reward, exploration bonus slower than the update of $Q$ function, therefore give the agent more time to explore the optimal policy for maximizing a certain intrinsic reward $r_{\text{int}}(s, a)$. After updating the parameter $\bphi_{t\% N}$, we perform a soft update for the target network as
\begin{align}
    \bar \bphi_{t \% N} \leftarrow (1 - \eta) \bar \bphi_{t \% N} + \eta \bphi_{t \% N}, \label{eqq:3}
\end{align}
where we follow the setting in URLB to set $\eta = 0.01$. After updating the $Q$ function, the algorithm then updates the actor $\pi_{\btheta}(a | s)$ following DDPG in maximizing
\begin{align}
    \cL(\btheta) = \sum_{(s, a, s') \in \cB}\sum_{i \in [N]} Q_i(s, \pi_{\btheta}(a | s))\label{eqq:4}
\end{align}

We summarize the exploration algorithm in Algorithm~\ref{alg:prac}, in particular, we use Adam to optimize the loss function defined by~\eqref{eqq:1} and~\eqref{eqq:3}.
\begin{algorithm}[h]
\caption{\alg - Exploration Phase -- Implementation}\label{alg:prac}
\begin{algorithmic}[1]
\REQUIRE Number of ensemble $N$, update speed $\eta$, exploration step $T$, (reward-free) environment \texttt{env}, 
\REQUIRE Action variance $\sigma^2$, minibatch size $B$, exploration bonus $\beta$, discount factor $\gamma$
\STATE For all $i \in [N]$, initialize $\bphi_i$, let $\bar \bphi_i \leftarrow \bphi_i$
\STATE Initialize policy network $\pi_{\btheta}$, replay buffer $\cD = \emptyset$
\STATE Observe initial state $s_1$
\FOR {$t = 1, \cdots, T$}
\STATE Sample $\zeta \sim \text{Unif.}[0, 1]$, sample $a_t \sim \Big\{N(\pi(\cdot | s_t), \sigma^2)$ \textbf{if } $\zeta \le 1 - \epsilon$ \textbf{else }  $\text{Unif.} (\cA)\Big\}$
\STATE Observe $s_{t+1}$, let $\cD \leftarrow \cD \cup (s_t, a_t, s_{t+1})$
\STATE \textbf{If} \texttt{env.done}, restart \texttt{env} and observe initial state $s_{t+1}$
\STATE Sample a minibatch $\cB = \{(s, a, s')\} \subseteq \cD$ with size $B$
\STATE For each $(s, a, s')$ triplet, calculate $\sigma^2(s, a), r_{\text{int}}(s, a), b(s, a)$ according to~\eqref{eqq:2}.
\STATE Update $Q$-network $Q_{t \% N}$ by taking one step minimizing $\cL(\bphi_{t \% N})$ according to~\eqref{eqq:1}
\STATE Update actor $\pi_{\btheta}(\cdot | s)$ by taking one step maximizing $\cL(\btheta)$ according to~\eqref{eqq:4}
\STATE Update target $Q$-network following~\eqref{eqq:3}
\ENDFOR
\end{algorithmic}
\end{algorithm}
\begin{algorithm}[h]
\caption{\alg - Planning Phase -- Implementation (DDPG)}\label{alg:prac-plan}
\begin{algorithmic}[1]
\REQUIRE Update speed $\eta$, training $K$, environment \texttt{env}, reward function $r(\cdot, \cdot)$
\REQUIRE Action variance $\sigma^2$, minibatch size $B$, discount factor $\gamma$, offline training data $\cD$
\STATE Initialize $\bphi$, let $\bar \bphi \leftarrow \bphi$
\STATE Initialize policy network $\pi_{\btheta}$
\STATE Update every $(s, a, s')$ in $\cD$ to $(s, a, s', r(s, a))$
\FOR {$k = 1, \cdots, K$}
\STATE Sample a minibatch $\cB = \{(s, a, s', r(s, a))\} \subseteq \cD$
\STATE Calculate $\cL(\bphi) = \sum_{(s, a, s') \in \cB} \bigg(Q_{\bphi}(s, a) - \Big(r(s, a) + \gamma Q_{\text{target}}(s', \pi_{\btheta}(s'))\Big)\bigg)^2$
\STATE Update $Q$-network $Q_{t \% N}$ by taking one step minimizing $\cL(\bphi)$
\STATE Calculate actor loss $\cL(\btheta) = \sum_{(s, a, s') \in \cB} Q_{\bphi}(s, \pi_{\btheta}(a | s))$
\STATE Update actor $\pi_{\btheta}(\cdot | s)$ by taking one step maximizing $\cL(\btheta)$
\STATE Update target $Q$-network by $\bar \bphi \leftarrow (1 - \eta) \bar \bphi + \eta \bphi$
\ENDFOR
\end{algorithmic}
\end{algorithm}
\subsection{Details of offline training algorithm}
After collecting the dataset $\cD$, we call a reward oracle to label the reward $r$ for any triplet $(s, a, s') \in \cD$. Then the DDPG algorithm is called to learn the optimal policy. For the fair comparison with other benchmark algorithm, we do not add weighted regression in the planning phase, thus the algorithm stays the same with the one presented in URLB, as stated in Algorithm~\ref{alg:prac-plan}

\subsection{Hyper-parameters}

We present a common set of hyper-parameters used in our experiments in Table~\ref{tab:hyper-common}. And we list individual hyper-parameters for each method in table~\ref{tab:hyper-per}. All common hyper-parameters and individual hyper-parameters for baseline algorithms are the same as what is used in \citet{laskin2021urlb} and its implementations.

\begin{table}[t]
\centering
\caption{The common set of hyper-parameters.}
\small{
\begin{tabular}{lc}
\hline
\textbf{Common hyper-parameter} & \textbf{Value} \\
\hline
Replay buffer capacity & $10^6$ \\
Action repeat & 1 \\
n-step returns & 3 \\
Mini-batch size & 1024  \\
Discount ($\gamma$) & 0.99 \\
Optimizer & Adam \\
Learning rate & $10^{-4}$ \\
Agent update frequency & 2 \\
Critic target EMA rate ($\tau_Q$) & 0.01 \\
Features dim. & 50  \\
Hidden dim. & 1024 \\
Exploration stddev clip & 0.3 \\
Exploration stddev value & 0.2 \\
Number of frames per episode &  $1 \times 10^3$  \\
Number of online exploration frames & up to $ 1 \times 10^6$ \\
Number of offline planning frames & $1 \times 10^5$ \\
Critic network & $(|O| + |A|) \rightarrow 1024 \rightarrow \text{LayerNorm} \rightarrow \text{Tanh} \rightarrow 1024 \rightarrow \text{RELU} \rightarrow 1$ \\
Actor network &  $ |O| \rightarrow 50 \rightarrow \text{LayerNorm} \rightarrow \text{Tanh} \rightarrow 1024 \rightarrow \text{RELU} \rightarrow \text{action dim}$ \\
\hline
\end{tabular}
}
\label{tab:hyper-common}
\end{table}

\begin{table}[t]
\centering
\caption{Hyper-parameters of each algorithm.}
\small{%
\begin{tabular}{lc}
\hline
GFA-RFE & Value \\
\hline
Ensemble size & 10 \\
Exploration bonus & 2 \\
Exploration $\epsilon$ & 0.2 \\
\hline
ICM hyper-parameter & Value \\
\hline
Reward transformation & $\log(r + 1.0)$ \\
Forward net arch. & $(|O| + |A|) \rightarrow 1024 \rightarrow 1024 \rightarrow |O|$ ReLU MLP \\
Inverse net arch. & $(2 \times |O|) \rightarrow 1024 \rightarrow |A| $ ReLU MLP \\
\hline
Disagreement hyper-parameter &  Value \\
\hline
Ensemble size & 5 \\
Forward net arch: & $(|O| + |A|) \rightarrow 1024 \rightarrow 1024 \rightarrow |O| $ ReLU MLP \\
\hline
RND hyper-parameter &  Value \\
\hline
Representation dim. & 512 \\
Predictor \& target net arch. & $|O| \rightarrow 1024 \rightarrow 1024 \rightarrow 512 $ ReLU MLP \\
Normalized observation clipping & 5 \\
\hline
APT hyper-parameter &  Value \\
\hline
Representation dim. & 512 \\
Reward transformation & $\log(r + 1.0)$ \\
Forward net arch. & $(512 + |A|) \rightarrow 1024 \rightarrow 512 $ ReLU MLP \\
Inverse net arch. & $(2 \times 512) \rightarrow 1024 \rightarrow |A| $ ReLU MLP \\
k in NN & 12 \\
Avg top k in NN & True \\
\hline
SMM hyper-parameter &  Value \\
\hline
Skill dim. & 4 \\
Skill discrim lr & $10^{-3}$ \\
VAE lr & $10^{-2}$ \\
\hline
DIAYN hyper-parameter &  Value\\
\hline
Skill dim & 16 \\
Skill sampling frequency (steps) & 50 \\
Discriminator net arch. & $512 \rightarrow 1024 \rightarrow 1024 \rightarrow 16 $ ReLU MLP \\
\hline
APS hyper-parameter &  Value\\
\hline
Reward transformation & $\log(r + 1.0)$ \\
Successor feature dim. & 10 \\
Successor feature net arch. & $|O| \rightarrow 1024 \rightarrow 10 $ ReLU MLP \\
k in NN & 12 \\
Avg top k in NN & True \\
Least square batch size & 4096 \\
\hline
\end{tabular}%
}
\label{tab:hyper-per}
\end{table}

\subsection{Ablation Study}
\subsubsection{Learning Processes}

Figures~\ref{fig:subfigs1} and \ref{fig:subfigs2} illustrate the episode rewards for each algorithm across training steps for various tasks, demonstrating that the performance of our algorithm (Algorithm~\ref{alg:main}) ranks among the top tier in all tasks.

\begin{figure}[b]
 \centering
 
\begin{subfigure}{0.49\textwidth}
    \includegraphics[width=\textwidth]{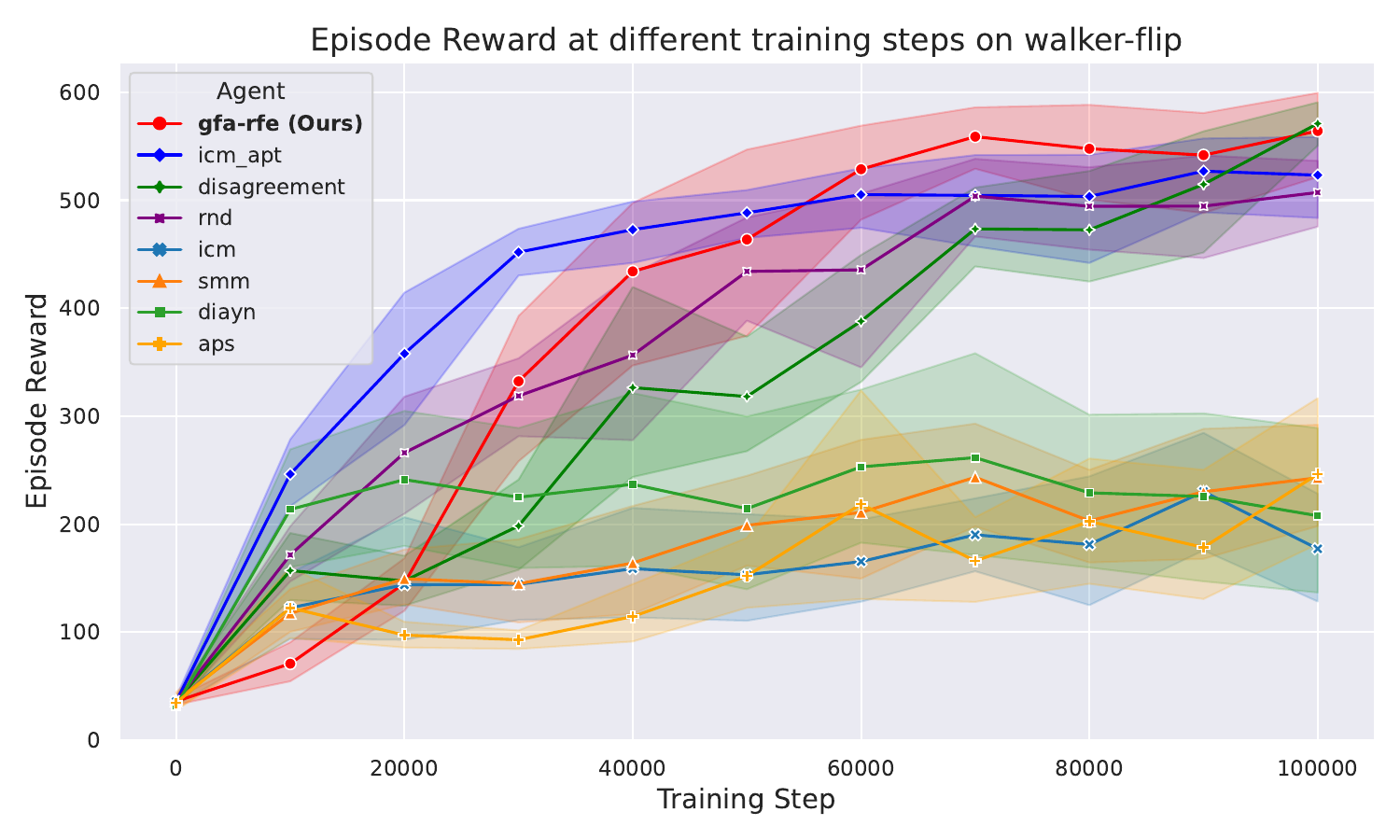}
    \caption{Flip}
    \label{fig:1first}
\end{subfigure}
\hfill
\begin{subfigure}{0.49\textwidth}
    \includegraphics[width=\textwidth]{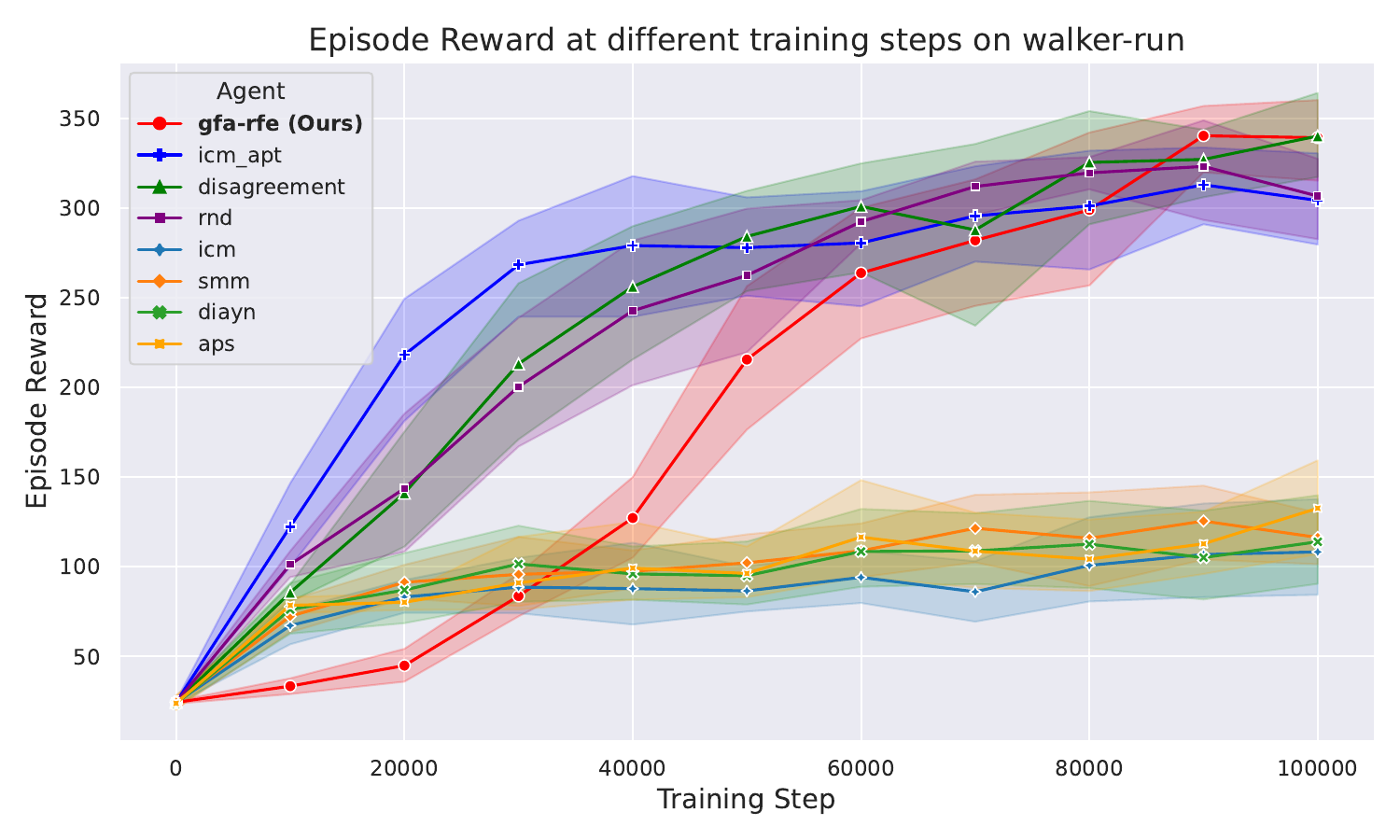}
    \caption{Run}
    \label{fig:1second}
\end{subfigure}
\hfill
\begin{subfigure}{0.49\textwidth}
    \includegraphics[width=\textwidth]{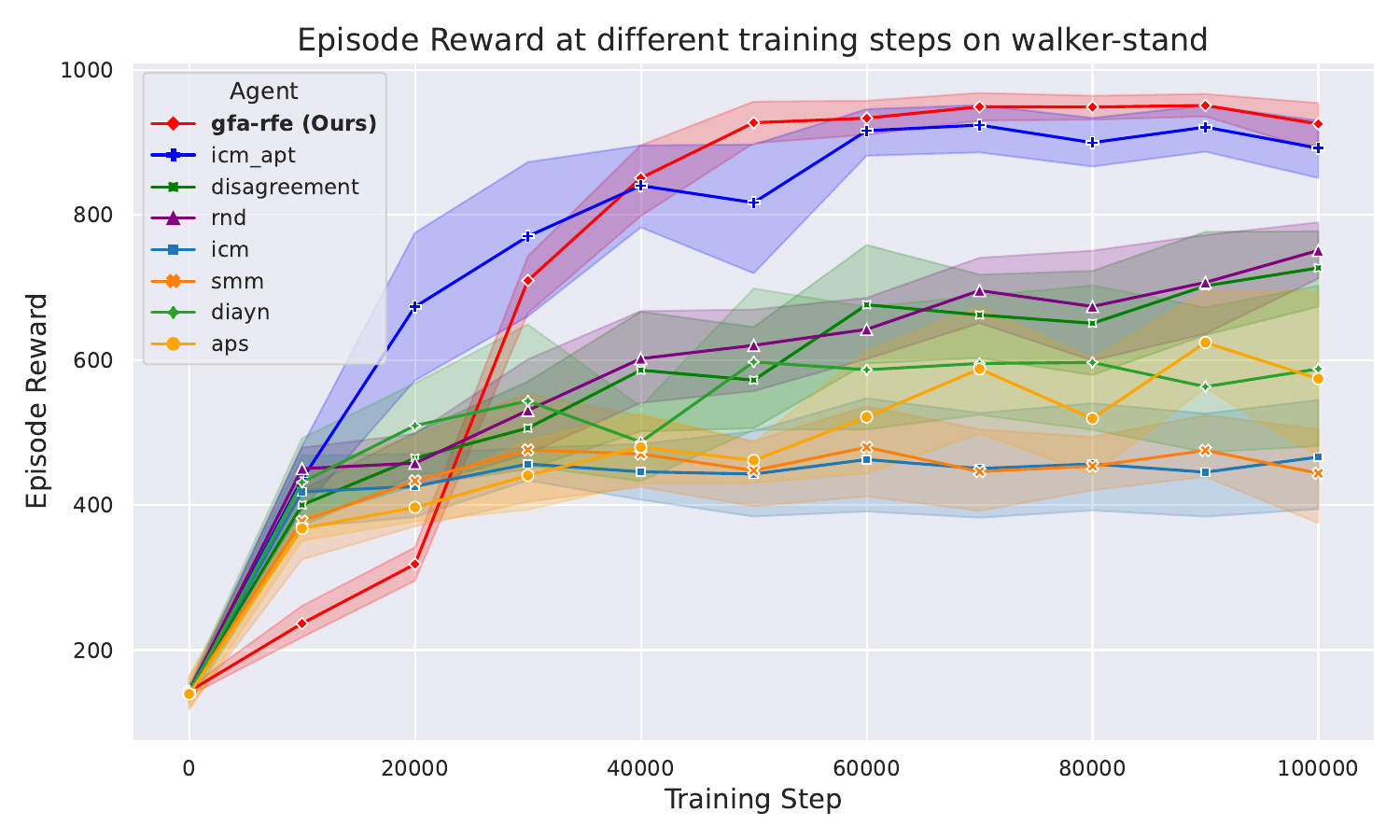}
    \caption{Stand}
    \label{fig:1third}
\end{subfigure}
\hfill
\begin{subfigure}{0.49\textwidth}
    \includegraphics[width=\textwidth]{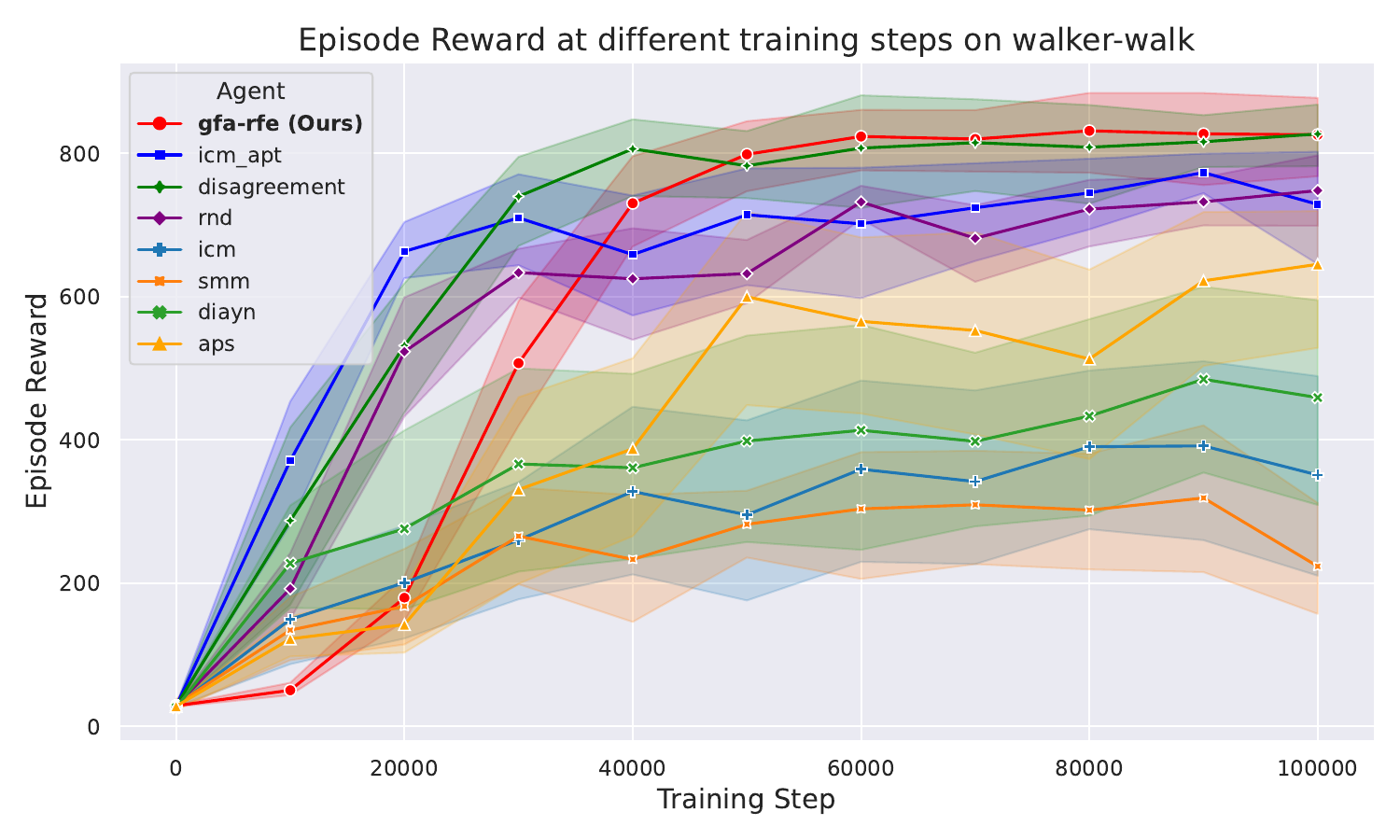}
    \caption{Walk}
    \label{fig:1fourth}
\end{subfigure}
  \caption{Episode reward at different offline training steps for different tasks for the \emph{walker} environment:~\eqref{fig:1first}: \emph{walker-flip};~\eqref{fig:1second}: \emph{walker-run};~\eqref{fig:1third} \emph{walker-stand};~\eqref{fig:1fourth} \emph{walker-walk}.} \label{fig:subfigs1}
  \end{figure}

  \begin{figure}[b]
  \begin{subfigure}{0.49\textwidth}
        \includegraphics[width=\textwidth]{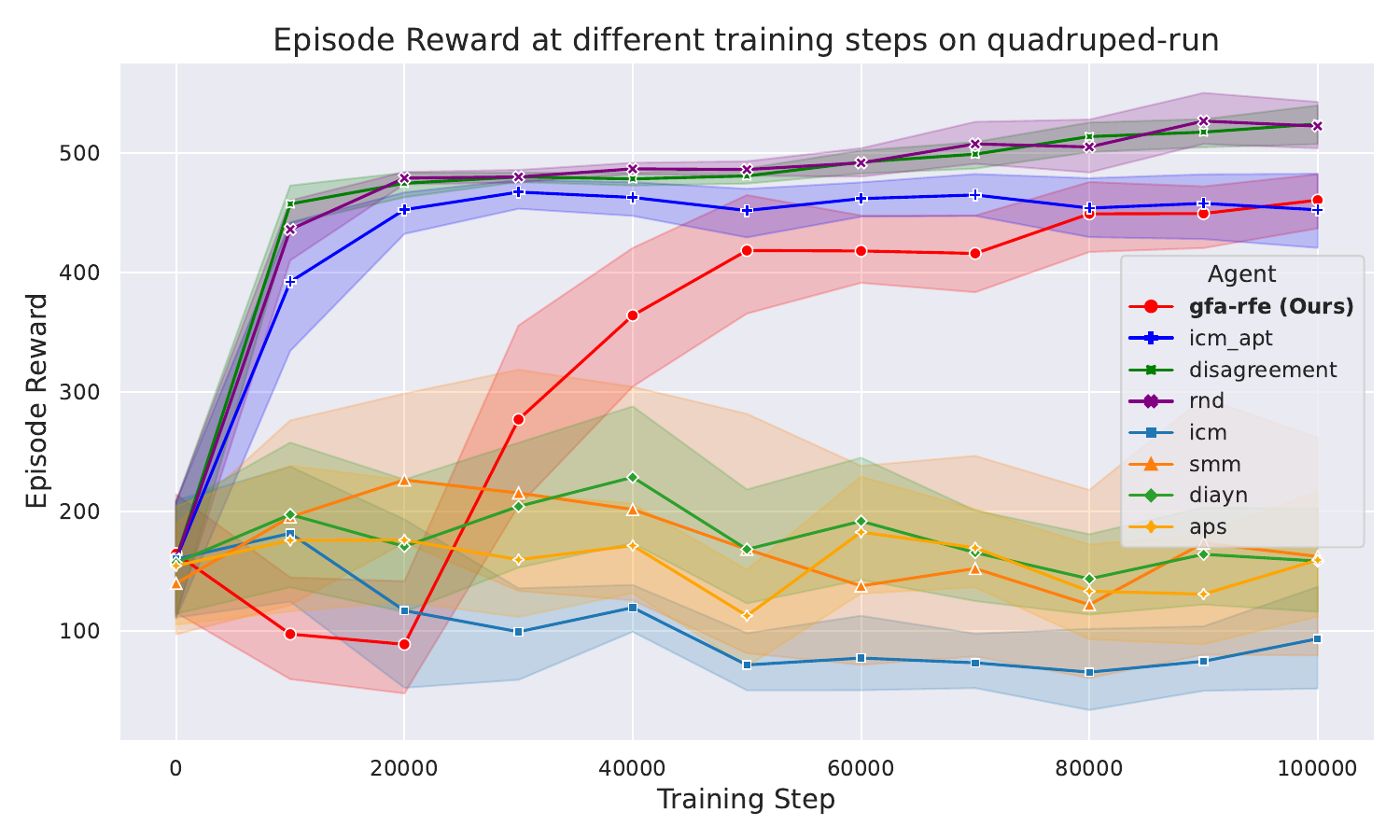}
    \caption{Run}
    \label{fig:2first}
\end{subfigure}
\hfill
\begin{subfigure}{0.49\textwidth}
    \includegraphics[width=\textwidth]{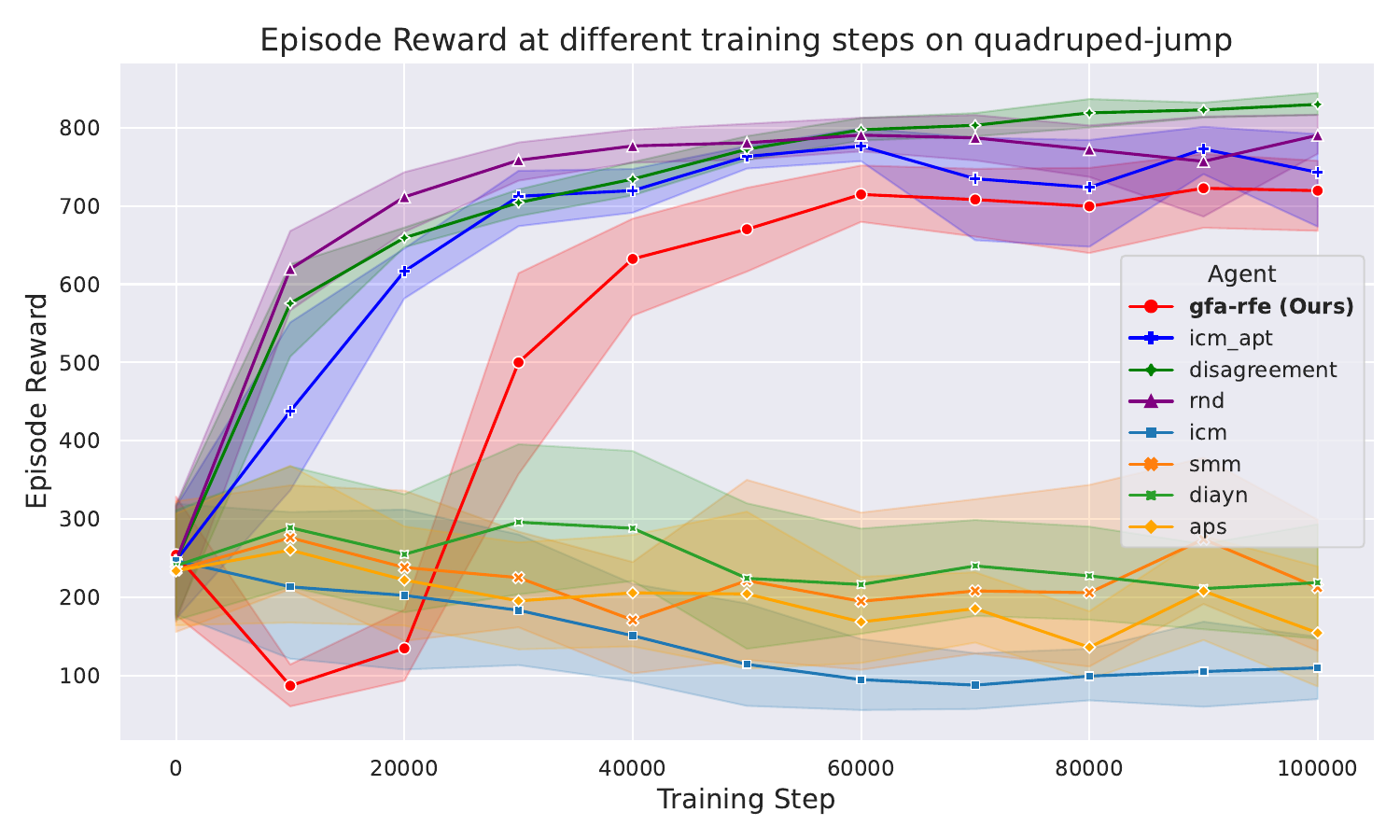}
    \caption{Jump}
    \label{fig:2second}
\end{subfigure}
\hfill
\begin{subfigure}{0.49\textwidth}
    \includegraphics[width=\textwidth]{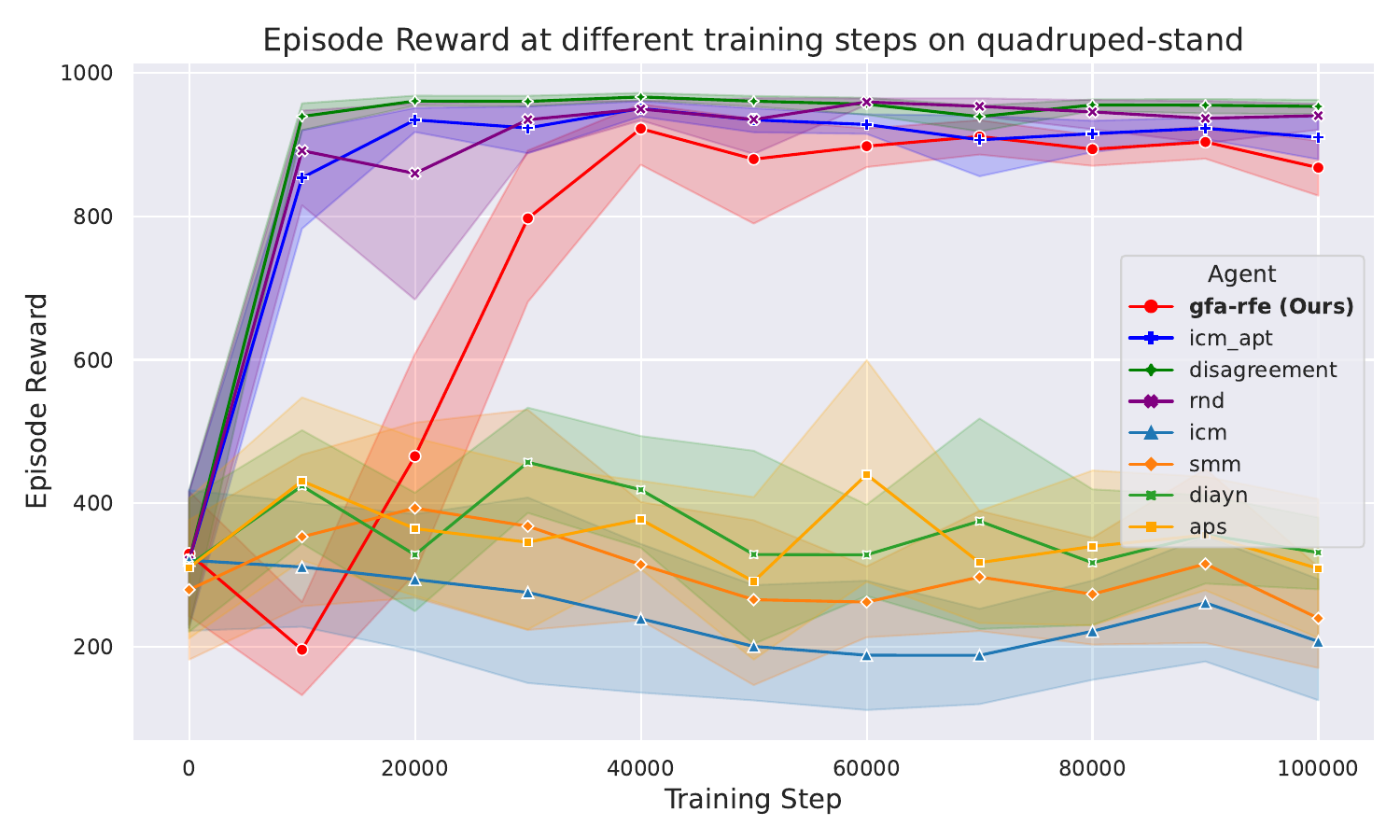}
    \caption{Stand}
    \label{fig:2third}
\end{subfigure}
\hfill
\begin{subfigure}{0.49\textwidth}
    \includegraphics[width=\textwidth]{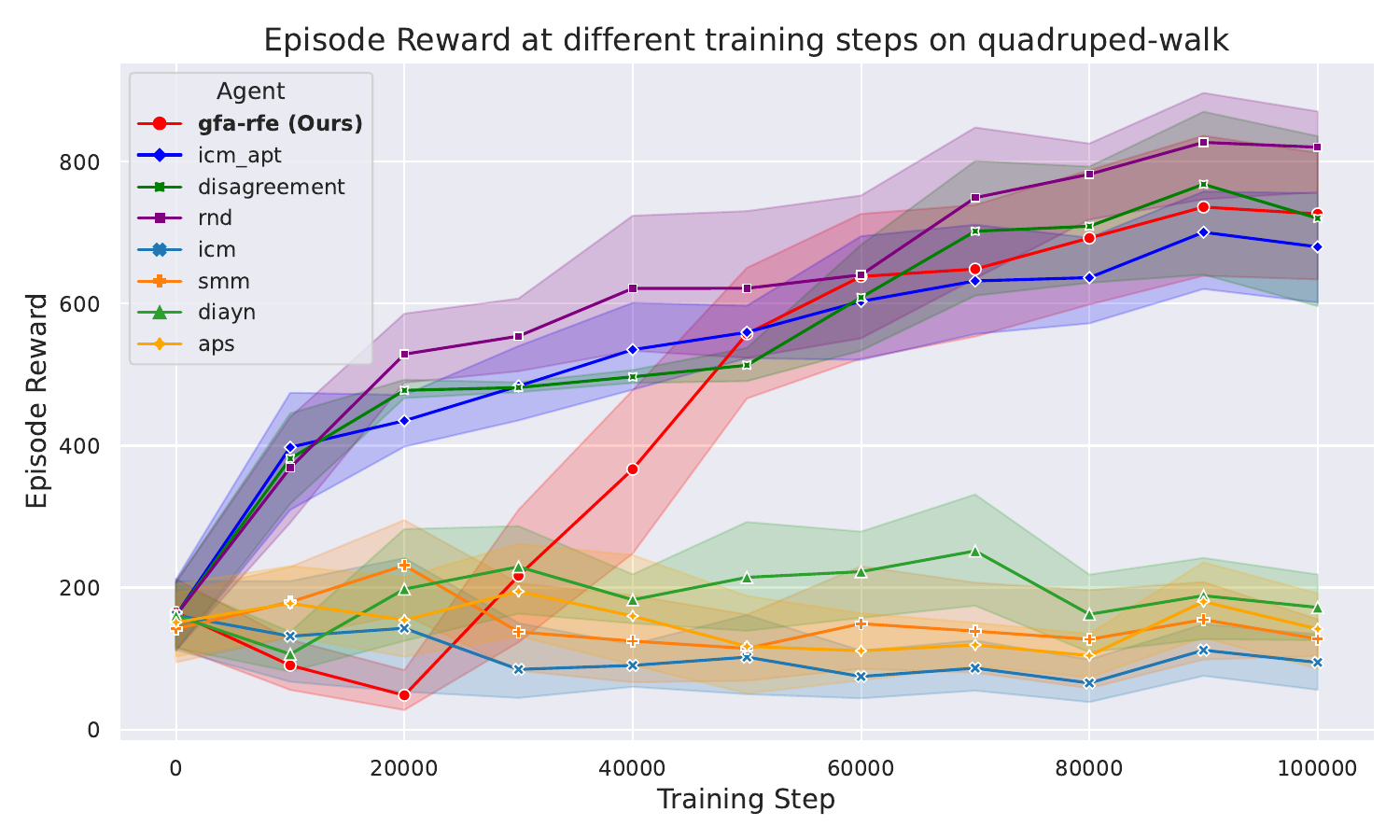}
    \caption{Walk}
    \label{fig:2fourth}
\end{subfigure}
  \caption{Episode reward at different offline training steps for different tasks for the \emph{quadruped} environment:~\eqref{fig:2first}: \emph{quadruped-flip};~\eqref{fig:2second}: \emph{quadruped-run};~\eqref{fig:2third} \emph{quadruped-stand};~\eqref{fig:2fourth} \emph{quadruped-walk}.}  \label{fig:subfigs2}
\end{figure}

\subsubsection{Numbers of Exploration Episodes}

Figures~\ref{fig:subfigs3} and \ref{fig:subfigs4} show the episode rewards for top-performing algorithms, including our algorithm (\alg), RND, Disagreement, and APT, across varying numbers of exploration episodes for different tasks. Notably, \alg~competes with these leading unsupervised algorithms effectively, matching their performance across a range of exploration episodes.

\begin{figure}[b]
 \centering
 
\begin{subfigure}{0.49\textwidth}
    \includegraphics[width=\textwidth]{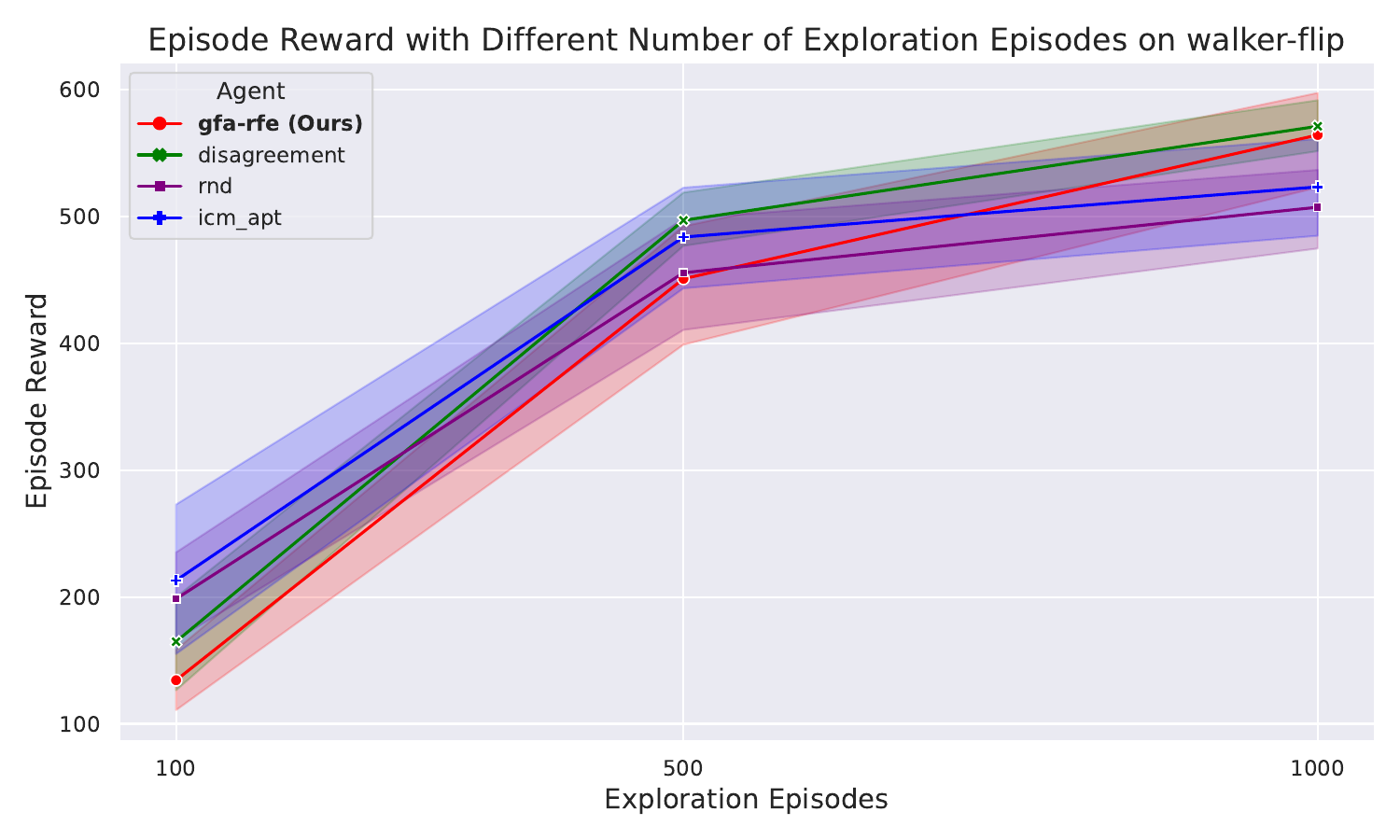}
    \caption{Flip}
    \label{fig:3first}
\end{subfigure}
\hfill
\begin{subfigure}{0.49\textwidth}
    \includegraphics[width=\textwidth]{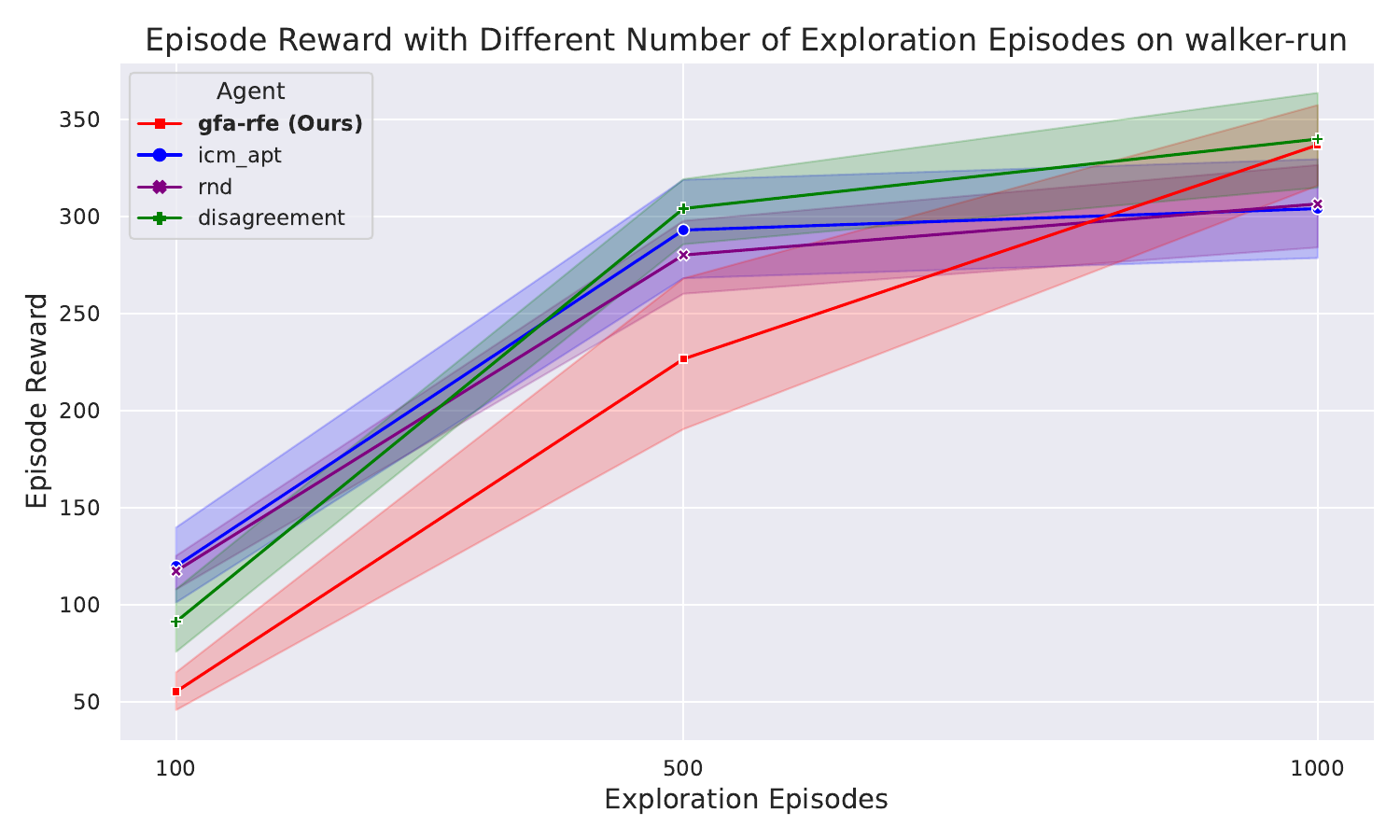}
    \caption{Run}
    \label{fig:3second}
\end{subfigure}
\hfill
\begin{subfigure}{0.49\textwidth}
    \includegraphics[width=\textwidth]{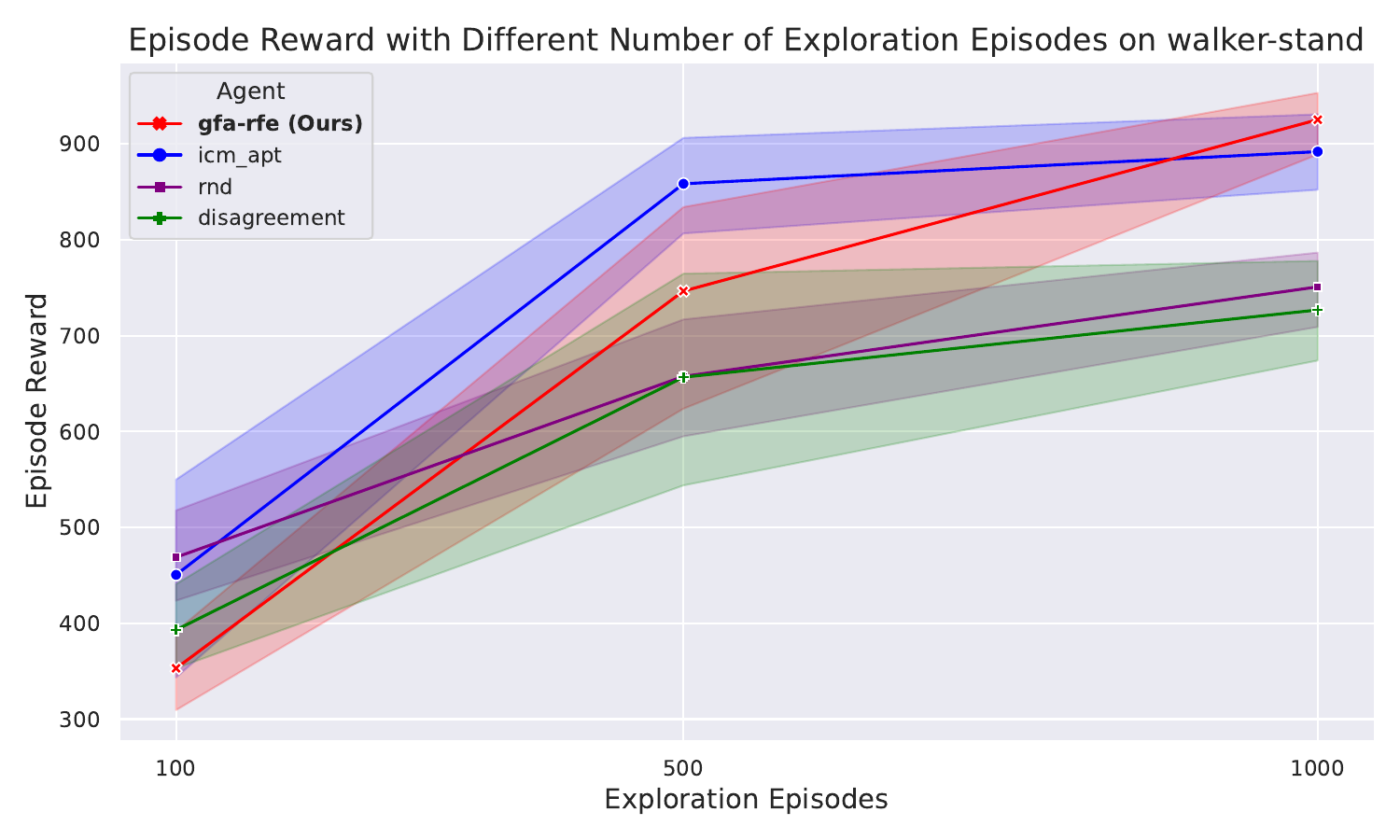}
    \caption{Stand}
    \label{fig:3third}
\end{subfigure}
\hfill
\begin{subfigure}{0.49\textwidth}
    \includegraphics[width=\textwidth]{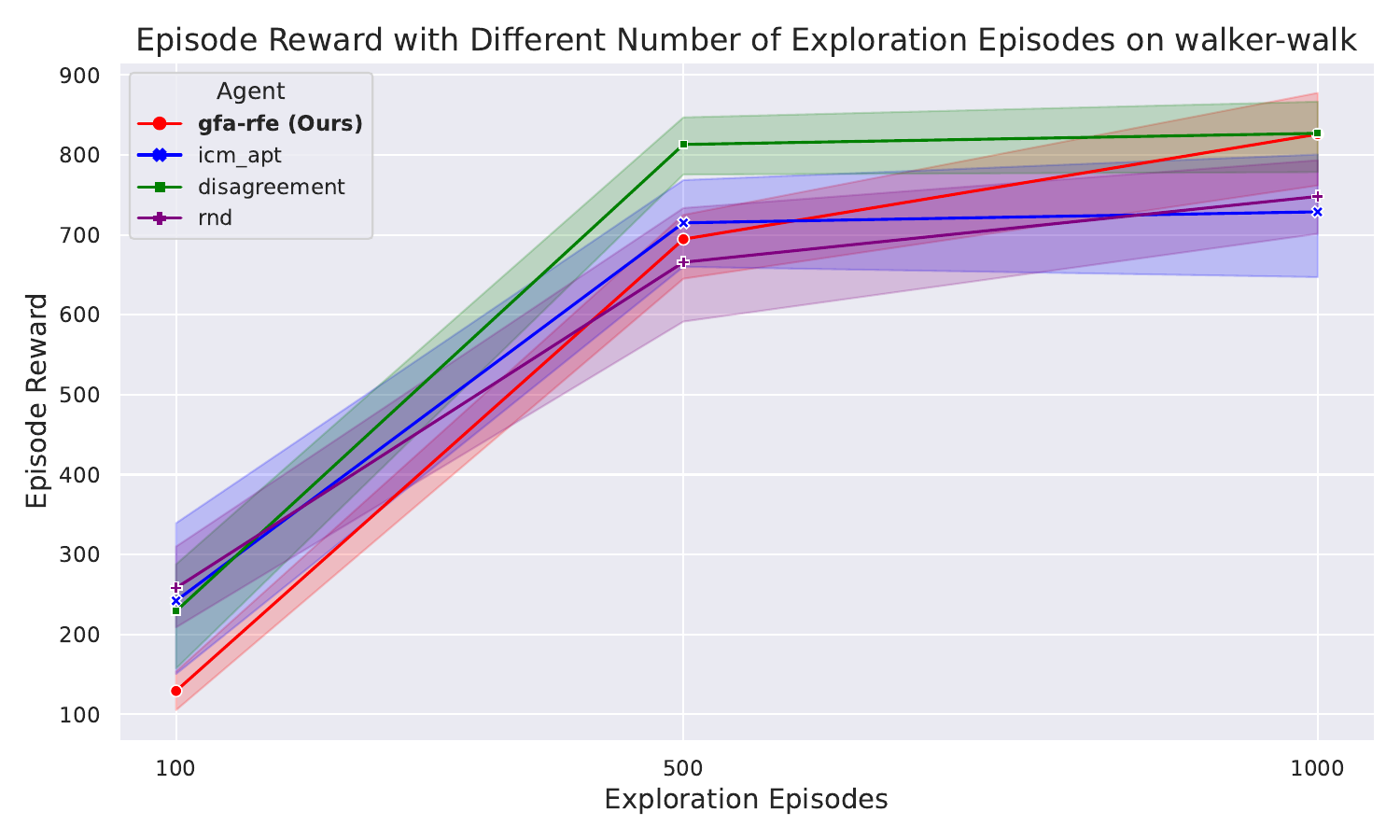}
    \caption{Walk}
    \label{fig:3fourth}
\end{subfigure}
  \caption{Episode reward with different numbers of exploration episodes for different tasks for the \emph{walker} environment:~\eqref{fig:3first}: \emph{walker-flip};~\eqref{fig:3second}: \emph{walker-run};~\eqref{fig:3third} \emph{walker-stand};~\eqref{fig:3fourth} \emph{walker-walk}.} \label{fig:subfigs3}
  \end{figure}

  \begin{figure}[b]

  \begin{subfigure}{0.49\textwidth}
    \includegraphics[width=\textwidth]{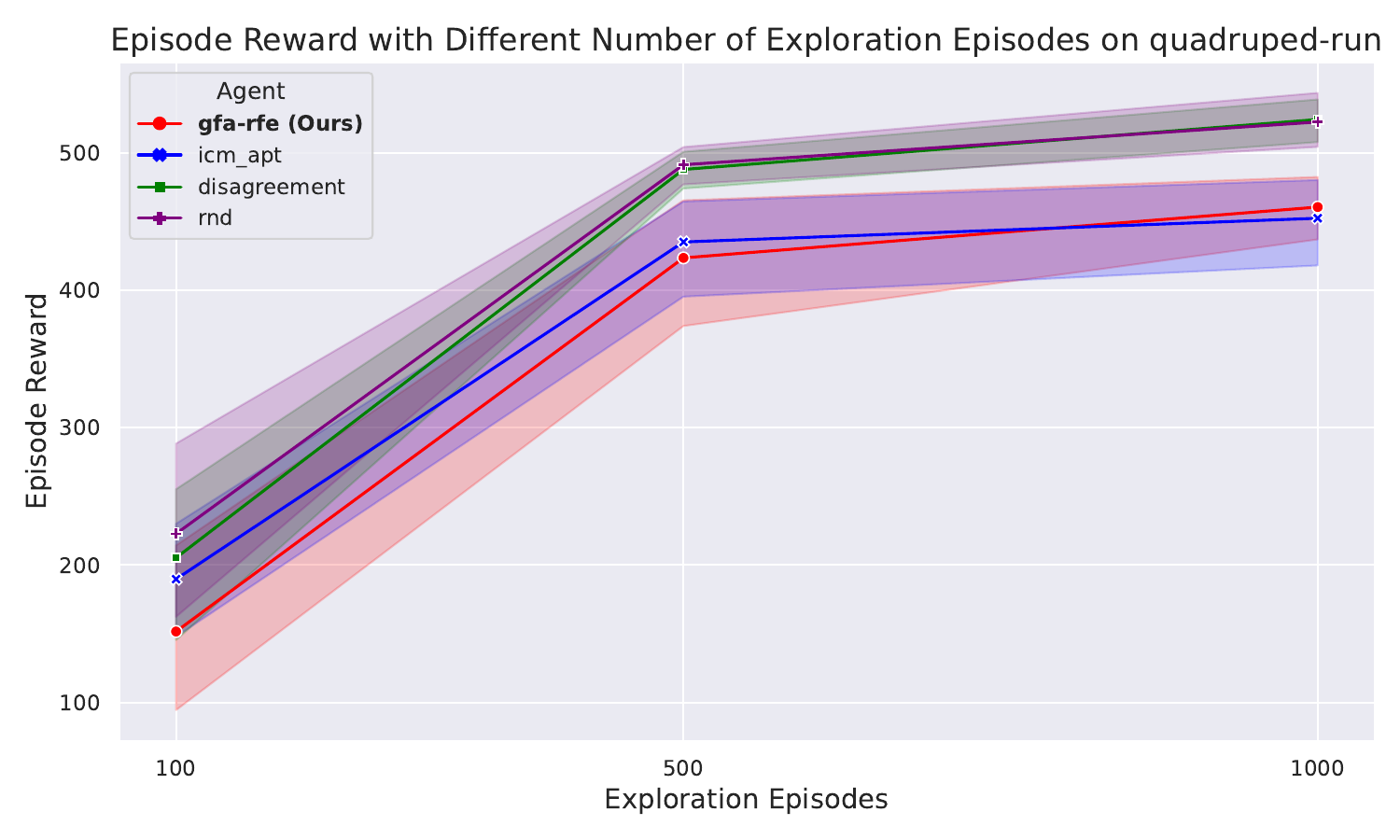}
    \caption{Run}
    \label{fig:4first}
\end{subfigure}
\hfill
\begin{subfigure}{0.49\textwidth}
    \includegraphics[width=\textwidth]{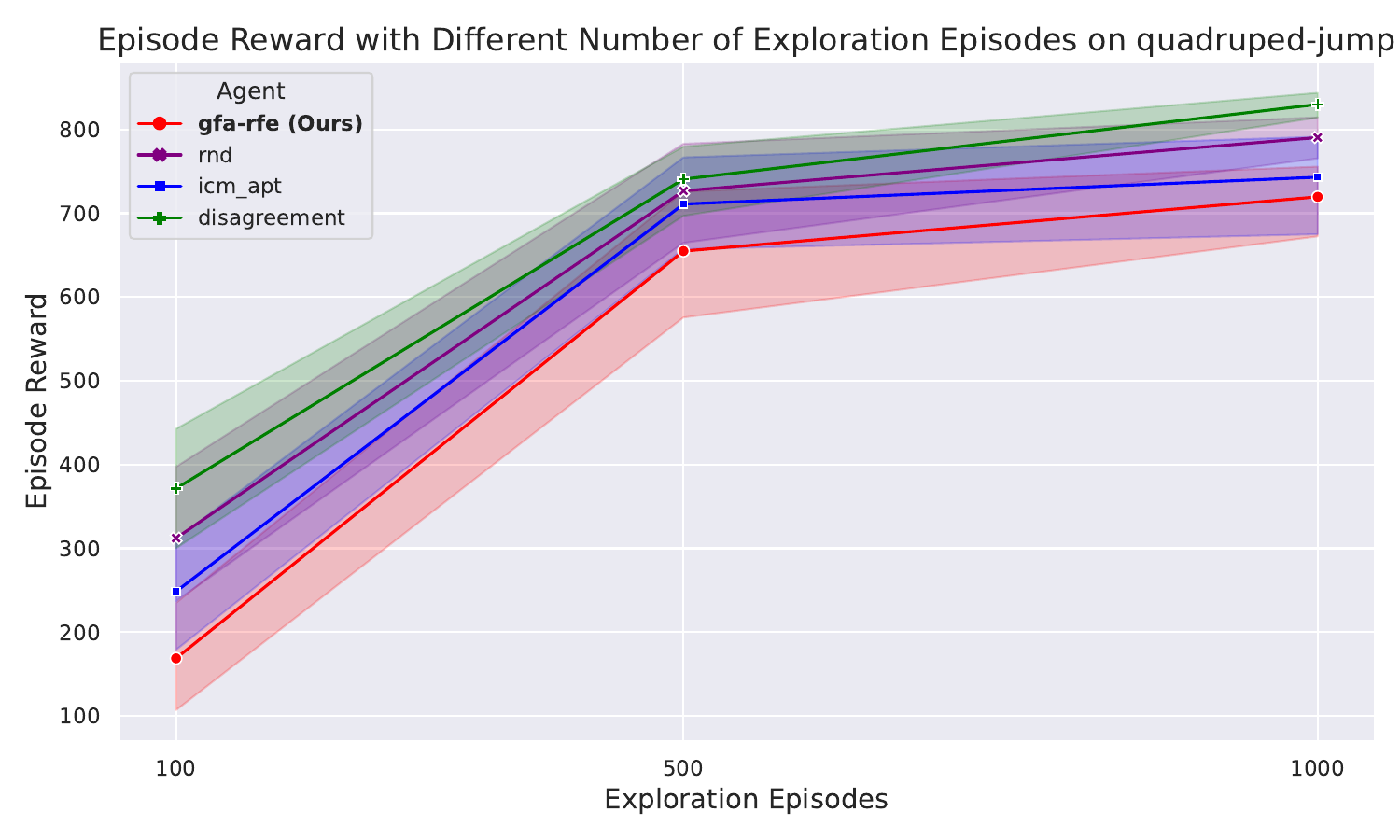}
    \caption{Jump}
    \label{fig:4second}
\end{subfigure}
\hfill
\begin{subfigure}{0.49\textwidth}
    \includegraphics[width=\textwidth]{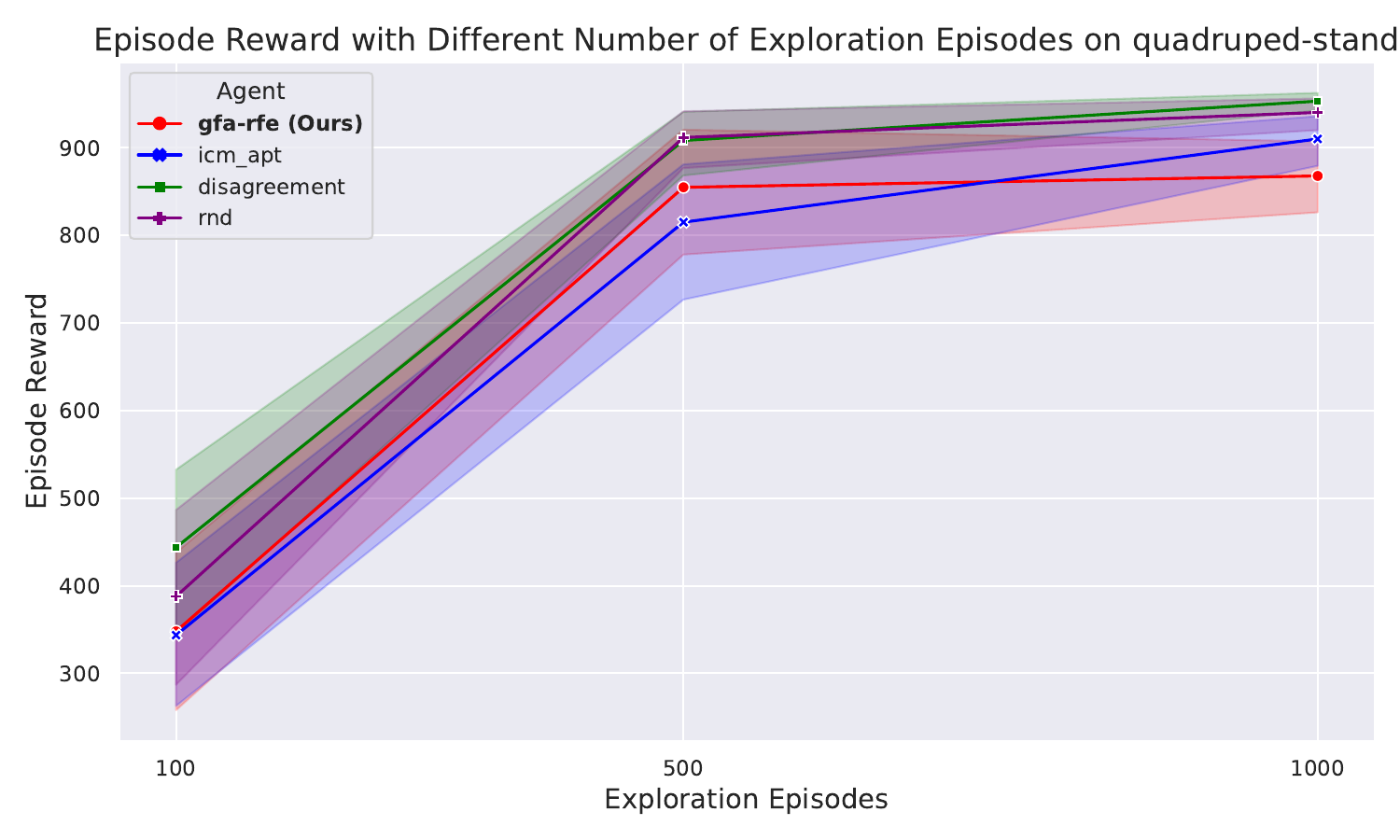}
    \caption{Stand}
    \label{fig:4third}
\end{subfigure}
\hfill
\begin{subfigure}{0.49\textwidth}
    \includegraphics[width=\textwidth]{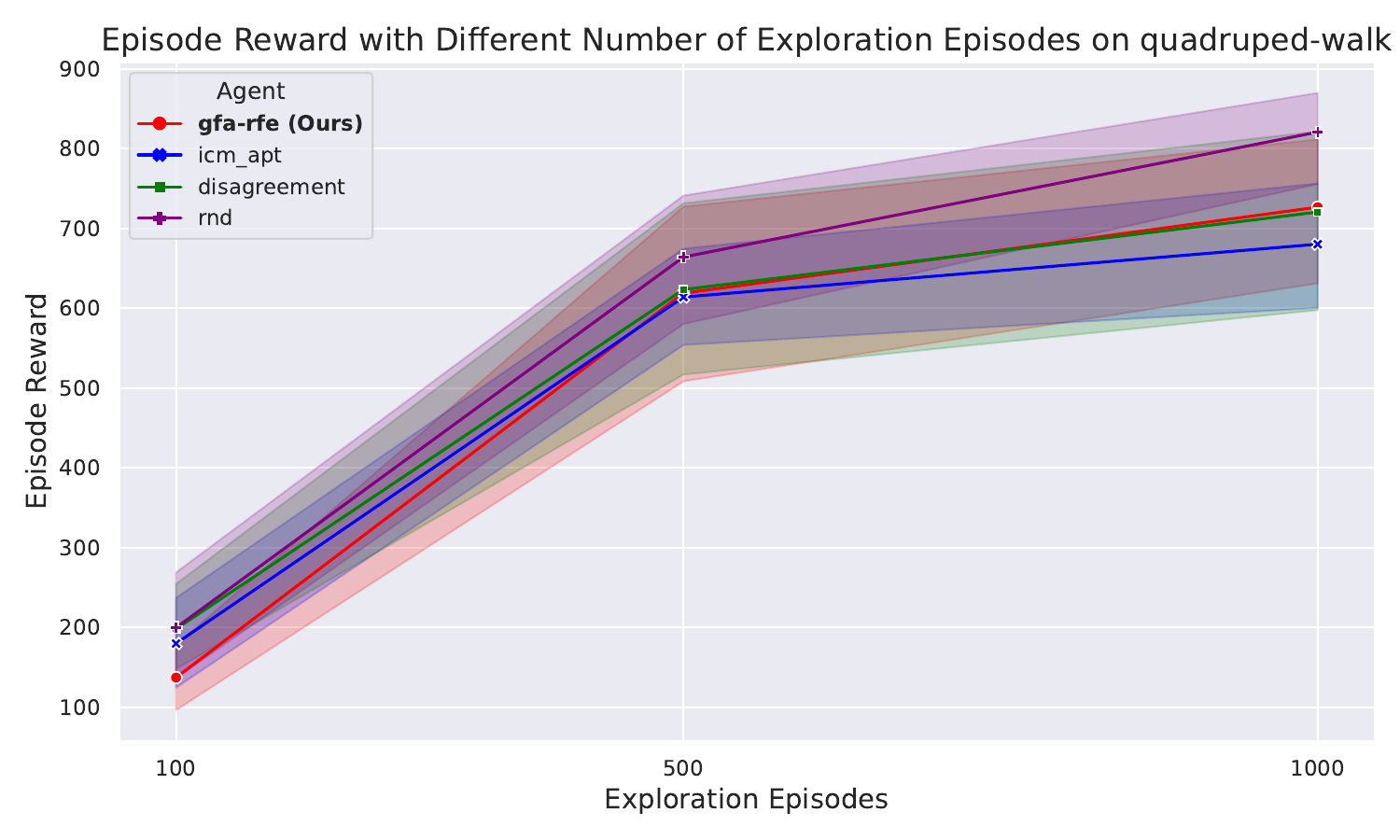}
    \caption{Walk}
    \label{fig:4fourth}
\end{subfigure}
  \caption{Episode reward with different numbers of exploration episodes for different tasks for the \emph{quadruped} environment:~\eqref{fig:4first}: \emph{quadruped-flip};~\eqref{fig:4second}: \emph{quadruped-run};~\eqref{fig:4third} \emph{quadruped-stand};~\eqref{fig:4fourth} \emph{quadruped-walk}.}  \label{fig:subfigs4}
\end{figure}

\end{document}